\documentclass[journal]{IEEEtran}

\IEEEoverridecommandlockouts
\usepackage{caption}
\usepackage{subcaption}
\usepackage[nospace]{cite}
\usepackage{amsmath,amssymb,amsfonts}
\usepackage{amsthm}
\usepackage{graphicx}
\usepackage{hyperref}
\usepackage[ruled,vlined]{algorithm2e}
\usepackage{xcolor}
\usepackage{float}

\newtheorem{theorem}{Theorem}
\newtheorem{lemma}{Lemma}

\newtheorem{remark}{Remark}

\newenvironment{customdefinition}[1]{%
\par\vspace{\baselineskip}\noindent\textbf{Definition #1:}\quad\ignorespaces
}{%
\par\vspace{\baselineskip}\ignorespacesafterend
}

\DeclareMathOperator*{\argmin}{arg\,min}
\usepackage{textcomp}
\makeatletter
\newcommand{\distas}[1]{\mathbin{\overset{#1}{\kern\z@\sim}}}%
\newsavebox{\mybox}\newsavebox{\mysim}
\newcommand{\distras}[1]{%
  \savebox{\mybox}{\hbox{\kern3pt$\scriptstyle#1$\kern3pt}}%
  \savebox{\mysim}{\hbox{$\sim$}}%
  \mathbin{\overset{#1}{\kern\z@\resizebox{\wd\mybox}{\ht\mysim}{$\sim$}}}%
}
 \def\BibTeX{{\rm B\kern-.05em{\sc i\kern-.025em b}\kern-.08em
     T\kern-.1667em\lower.7ex\hbox{E}\kern-.125emX}}

\begin{document}

\title{Smoothing ADMM for Sparse-Penalized Quantile Regression with Non-Convex Penalties}

\author{Reza Mirzaeifard~\IEEEmembership{Student Member,~IEEE}, Naveen K. D. Venkategowda~\IEEEmembership{Senior Member, ~IEEE},  Vinay Chakravarthi Gogineni~\IEEEmembership{Senior Member, ~IEEE}, Stefan Werner$^{\star} ~\IEEEmembership{Fellow, ~IEEE}$
 \thanks{This work was supported in part by the Research Council of Norway.}
\thanks{Stefan Werner and Reza Mirzaeifard are with the Department
of Electronic Systems, Norwegian University of Science and Technology-NTNU, Norway,
Trondheim, 7032 Norway (e-mail:\{stefan.werner, reza.mirzaeifard\}@ntnu.no).}
\thanks{Naveen K. D. Venkategowda was with the
Department of Science and Technology,  Linköping University, Sweden (e-mail:naveen.venkategowda@liu.se).}
\thanks{Vinay Chakravarthi Gogineni is with the SDU Applied AI and Data Science, The Maersk Mc-Kinney Moller Institute, University of Southern Denmark, Denmark (e-mail: vigo@mmmi.sdu.dk). }

 \thanks{Manuscript received April 19, 2023}
}

\maketitle

\begin{abstract}
This paper investigates quantile regression in the presence of non-convex and non-smooth sparse penalties, such as the minimax concave penalty (MCP) and smoothly clipped absolute deviation (SCAD). The non-smooth and non-convex nature of these problems often leads to convergence difficulties for many algorithms. While iterative techniques like coordinate descent and local linear approximation can facilitate convergence, the process is often slow. This sluggish pace is primarily due to the need to run these approximation techniques until full convergence at each step, a requirement we term as a \emph{secondary convergence iteration}.
To accelerate the convergence speed, we employ the alternating direction method of multipliers (ADMM) and introduce a novel single-loop smoothing ADMM algorithm with an increasing penalty parameter, named SIAD, specifically tailored for sparse-penalized quantile regression. We first delve into the convergence properties of the proposed SIAD algorithm and establish the necessary conditions for convergence. Theoretically, we confirm a convergence rate of $o\big({k^{-\frac{1}{4}}}\big)$ for the sub-gradient bound of augmented Lagrangian. Subsequently, we provide numerical results to showcase the effectiveness of the SIAD algorithm. Our findings highlight that the SIAD method outperforms existing approaches, providing a faster and more stable solution for sparse-penalized quantile regression.
\end{abstract}

\section{Introduction}
\label{sec:intro}

Regression algorithms, such as linear regression, primarily aim to estimate the conditional mean of a response variable associated with a set of observations \cite{seber2012linear}. However, mean-based regression is notably sensitive to outliers and falls short in relating the response variable to another point, or range, within the conditional distribution, for instance, the median or a certain percentile. When data distributions exhibit heavy tails, the shortcomings of mean-based regression become even more pronounced. As an alternative, quantile regression provides a more comprehensive understanding of the underlying relationships between predictor and response variables based on quantiles, making it highly suitable for a diverse range of applications, particularly those involving heavy-tailed distributions. For instance, quantile regression has been used to predict regional wind power \cite{yu2020probabilistic}, where forecasts of extreme quantiles help manage the variability and intermittency associated with wind speeds, to estimate uncertainty in smart meter data \cite{taieb2016fore}, which allows for more robust modeling of electricity consumption patterns and better anomaly detection, and to forecast load in smart grids \cite{happy2021stat}, where accurate predictions across different quantiles can help grid operators make informed decisions about capacity planning and demand response, ultimately enhancing the stability and efficiency of the power system.

Moreover, estimating sparse models is essential in numerous real-world scenarios, such as quantitative traits in genetics \cite{he2016regularized}, gene selection for microarray gene expression \cite{algamal2018gene}, robust risk management models in finance \cite{tibshirani2014adaptive}, and analyzing the relationship between environmental factors and species distribution in ecological studies \cite{chen2021quantile}. This has led to an upsurge in research interest in the field \cite{wu2009variable, xue2012positive}. While $l_1$-penalized quantile regression has proven effective for highly sparse models, it exhibits poor performance and bias when model sparsity decreases due to the uniform shrinkage of all coefficients towards zero. To overcome this limitation, more sophisticated penalties, such as minimax concave penalty (MCP) \cite{fan2001variable} and smoothly clipped absolute deviation (SCAD) \cite{zhang2010nearly}, have been studied. Despite being non-convex and non-smooth, these penalties can selectively shrink model coefficients and mitigate the bias effect of the $l_1$-penalty, making them more suitable for a wider range of sparsity levels.

In penalized quantile regression with $l_1$ norm, a non-smooth yet convex problem, various optimization algorithms are available. These include linear programming (LP) \cite{belloni2011l1,koenker2005frisch}, sub-gradient methods \cite{wang2017distributed}, primal-dual methods \cite{ouyang2021lower}, and the alternating direction method of multipliers (ADMM) \cite{boyd2011distributed,gu2018admm}. LP is commonly used due to the compatibility of its linear constraints and objective function with the $l_1$ norm \cite{belloni2011l1,koenker2005frisch}. Despite the fact that the simplex method or interior point methods can lead to high computational complexity. Furthermore, LP's need to use approximation techniques to handle non-smoothness can potentially compromise accuracy. On the other hand, sub-gradient methods, which iteratively update the solution using the function's sub-gradients, can handle non-smoothness inherently and do not require the objective function to be differentiable \cite{wang2017distributed}. Despite this, these methods often show slower convergence rates than LP and require careful choice of step sizes. To improve convergence, gradient tracking methods, and smoothing techniques have been introduced \cite{wang2023smoothing,zhao2023accelerated}. Primal-dual methods, another alternative, transform the original problem into a saddle-point problem and tackle it using primal-dual updates. While they are generally more efficient than basic sub-gradient methods, they require careful tuning and present their own challenges due to the computational burden and the complexity of the saddle-point problem \cite{ouyang2021lower}. Lastly, the ADMM method, which breaks the problem into more manageable blocks, provides another solution \cite{boyd2011distributed}. It combines the principles of dual decomposition and augmented Lagrangian methods and has effectively solved many non-smooth convex constrained optimization problems \cite{boyd2011distributed}. In the context of $l_1$ penalized quantile regression, the ADMM variants, namely sparse coordinate descent ADMM (scdADMM) and proximal ADMM (pADMM), have been introduced \cite{gu2018admm}. They are used when one block of the primal update in the ADMM algorithm lacks a closed-form solution, prompting consideration of coordinate methods or proximal methods for block minimization.

 Dealing with non-smoothness is notably more challenging in non-convex settings such as penalized quantile regression with non-convex penalties, and this also limits the number of available optimization algorithms. Many conventional methods lose their convergence guarantees in these scenarios. Nevertheless, these methods can still provide convergence guarantees under specific conditions like lower semi-continuity, weak convexity, sharpness, or the presence of one smooth part \cite{chen2021distributed,zeng2022moreau,wang2019global,yashtini2020convergence,davis2018subgradient,swenson2022distributed}. A common strategy to approach non-smooth, non-convex problems is the majorization-minimization (MM) or the Local Linear Approximation (LLA) frameworks \cite{peng2015iterative,sun2016majorization}. These frameworks construct and minimize a surrogate function that majorizes the original non-convex function, designed to have desirable properties like smoothness or convexity, making optimization more tractable  \cite{sun2016majorization}. By incorporating the idea of MM with coordinate descent, the iterative coordinate descent algorithm (QICD) was proposed \cite{peng2015iterative}, but it suffers from high computational complexity and slow convergence rates. To improve efficiency, combining the LLA framework with scdADMM (LSCD) or pADMM (LPA) has been considered \cite{gu2018admm}. However, as both MM and LLA frameworks involve iterative processes, they may lead to slow convergence, prompting the consideration of single-loop algorithms, which can be more robust and faster.
A recently proposed sub-gradient algorithm can handle weakly convex functions, achieving a convergence rate of $O(K^{-\frac{1}{4}})$ to the $\frac{1}{K}$-stationary point based on the derivative of Moreau-envelope function \cite{davis2019stochastic}. This algorithm can be adapted for quantile regression penalized with MCP or SCAD \cite{mirzaeifard2023distributed}, but the result depends on the step size, and the convergence speed might not be efficient. Thus, there is ongoing research to find more robust and efficient solutions for non-smooth, non-convex optimization problems.

In light of the proven effectiveness of the ADMM algorithm, its application to quantile regression is quite appealing. However, implementing ADMM in non-convex scenarios with proven convergence remains challenging since existing non-convex ADMM methods frequently demand either a smooth part or an implicit Lipschitz condition to assure convergence \cite{wang2019global, hong2015convergence, yashtini2020convergence, themelis2020douglas,mirzaeifard2022robust}. Characteristics like Lipschitz differentiability can be beneficial in regulating the change in the dual update variable in non-convex optimization problems \cite{wang2019global}. In scenarios lacking convexity in the objective function, managing the change in the dual update step relative to the primal variables becomes essential for ensuring convergence. Precisely, setting bounds on the change in the dual update step, in accordance with the primal variables, could offer a means for parameter tuning and proof of convergence \cite{wang2019global,mirzaeifard2022robust}. However, in non-smooth and non-convex settings, such as sparse penalized quantile regression, the conditions for Lipschitz differentiability or implicit Lipschitz differentiability might not always be satisfied \cite{mirzaeifard2022admm}. Thus, a compelling need arises for developing enhanced ADMM-based optimization methods that can efficiently handle such conditions without relying on these assumptions.

This work proposes a smoothing ADMM algorithm with time-increasing penalty parameters referred to as SIAD to handle quantile regression problems with non-convex and non-smooth sparse penalties, such as MCP and SCAD. The SIAD algorithm tackles non-convex optimization problems by incorporating smoothing techniques without requiring a smooth part or an implicit Lipschitz condition. Our method transforms the non-smooth, non-convex problem into a series of smooth approximations, facilitating efficient and reliable convergence. The proposed SIAD employs a process where an upper-bound smooth function approximates the quantile regression function. This strategy allows us to regulate changes in the dual update step, ensuring the obtained limit converges to the actual quantile regression function. As the smoothness of the approximation function diminishes with each iteration, we apply a time-varying increasing penalty parameter to retain control over the dual variable. We also employed coordinate descent to update the block that can not admit a closed-form solution, resulting in a simpler algorithm and faster convergence. The contributions of this paper are as follows:
\begin{itemize}
\item We introduce the SIAD, a novel smoothing ADMM algorithm for sparse-penalized quantile regression with non-convex penalties, specifically focusing on MCP and SCAD penalties.
\item We overcome the challenges posed by the non-convex and non-smooth nature of the quantile regression function and penalties by employing an iterative smoothing process, enabling us to demonstrate convergence under mild conditions.
\item We provide a detailed analysis of the SIAD, demonstrating a convergence rate of $o\big({k^{-\frac{1}{4}}}\big)$ for the sub-gradient bound of augmented Lagrangian. 
\item Through extensive simulations, we validate the superior performance of our proposed algorithm in comparison to existing methods in terms of accuracy and convergence.
\end{itemize}
The remainder of this paper is organized as follows. Section \ref{sec3} outlines the necessary preliminaries for our study. In Section \ref{sec4}, we present our proposed SIAD algorithm. The convergence proof is given in Section \ref{sec5}. In Section \ref{sec6}, we conduct numerical simulations to validate the performance and effectiveness of our algorithm. Lastly, we draw conclusions and discuss future work in Section \ref{sec7}.

\noindent\textit{\textbf{Mathematical Notations}}:  
Bold letters $\mathbf{a}$ and $\mathbf{A}$ 
represent vectors and matrices, respectively. The transpose of
$\mathbf{A}$ is denoted as $\mathbf{A}^\text{T}$. The $j$-th column of a matrix $\mathbf{A}$ is denoted as $\mathbf{A}_{:,j}$, and the $j$-th element of a vector $\mathbf{x}$ is denoted as $x_j$. In addition, we let $\mathbf{A}_{<s} \mathbf{x}_{<s} := \sum_{i<s} \mathbf{A}_{:,i} x_i$ and, similarly, $\mathbf{A}_{>s} \mathbf{x}_{>s} := \sum_{i>s} \mathbf{A}_{:,i} x_i$. Moreover, for a function $h:\mathbb{R}^n \rightarrow \mathbb{R}$ and penalty parameter $\gamma>0$, the proximal function is defined as: $\textbf{Prox}_{h}\mathopen{}\left(w;\gamma\right)\mathclose{}= \argmin_x \mathopen{}\left \{h\mathopen{}\left(x\right)\mathclose{}+\frac{1}{2\gamma}\mathopen{}\left\|x-w\right\|\mathclose{}_2^2\right \}$. Furthermore, for a scalar variable $u$, and penalty parameter $\alpha$, $\text{Shrink}\mathopen{}\left(u,\alpha\right)\mathclose{}=\frac{u}{|u|}\max\{0,|u|-\alpha\}$. Finally, $\partial f\mathopen{}\left(u\right)\mathclose{}$ represents the sub-gradient of $f\mathopen{}\left(\cdot\right)\mathclose{}$ at $u$.
\section{Preliminaries}\label{sec3}
\subsection{Sparse Quantile Regression Framework}
Consider a scalar random variable $Y$ and a $P$-dimensional vector of predictors $\boldsymbol{\chi}$. We define the conditional cumulative distribution function as $F_{Y}\mathopen{}\left(y|\mathbf{x}\right)\mathclose{}=P\mathopen{}\left(Y\leq y|\boldsymbol{\chi}=\mathbf{x}\right)\mathclose{}$ and the $\tau$-th conditional quantile for $\tau \in \mathopen{}\left(0,1\right)\mathclose{}$ as $Q_Y\mathopen{}\left(\tau|\mathbf{x}\right)\mathclose{}=\inf\{y:F_Y\mathopen{}\left(y|\mathbf{x}\right)\mathclose{}\geq \tau\}$. The linear quantile regression model associates $Q_Y\mathopen{}\left(\tau|\mathbf{x}\right)\mathclose{}$ and $\mathbf{x}\in \mathbb{R}^P$ \cite{koenker1982robust}:
\begin{equation}\label{eq1}
Q_Y\mathopen{}\left(\tau|\mathbf{x}\right)\mathclose{}= \mathbf{x}^\text{T} \boldsymbol{\beta}_{\tau},
\end{equation}
where $\boldsymbol{\beta}_{\tau} \in \mathbb{R}^P$ denotes the parameters of the regression model that must be estimated.

Given a dataset comprising pairs $\{\mathbf{x}_i,y_i\}_{i=1}^{n}$ where $n$ is the number of samples and a specific value of $\tau$, the model parameter estimation can be obtained by solving the optimization problem \cite{koenker1982robust}:
\begin{equation}\label{eq2}
\hat{\mathbf{w}}=\argmin_{\mathbf{w}} \frac{1}{n} \sum_{i=1}^{n} \rho_{\tau}\mathopen{}\left(y_{i}-\mathbf{x}_{i}^{\text{T}}\mathbf{w}\right)\mathclose{},
\end{equation}
where $\mathbf{w}=\boldsymbol{\beta}_{\tau}$ and $\rho_{\tau}\mathopen{}\left(u\right)\mathclose{}=u \mathopen{}\left(\tau-I\mathopen{}\left(u<0\right)\mathclose{}\right)\mathclose{}$ is the check loss function.

Incorporating the penalty function $P_{\lambda,\gamma}\mathopen{}\left(\mathbf{w}\right)\mathclose{}$ allows leveraging \emph{a priori} information about the model coefficients, thus enhancing the inference quality. The penalized optimization problem \eqref{eq2} becomes
\begin{equation}\label{eq3}
\hat{\mathbf{w}}=\argmin_{\mathbf{w}} \frac{1}{n} \sum_{i=1}^{n} \rho_{\tau}\mathopen{}\left(y_{i}-\mathbf{{x}}_{i}^{\text{T}}\mathbf{w}\right)\mathclose{}+P_{\lambda,\gamma}\mathopen{}\left(\mathbf{w}\right)\mathclose{}.
\end{equation}
 When promoting sparsity, there are several penalty functions, or regularizers, to choose from, but the $l_1$ norm function has gained widespread popularity. Despite its popularity, the $l_1$ norm can lead to estimation bias and is not well-suited for group sparsity. In this paper, we present a solution that utilizes the MCP and SCAD penalty functions as $P_{\lambda,\gamma}\mathopen{}\left(\mathbf{w}\right)\mathclose{}=\sum_{p=1}^P g_{\lambda,\gamma}\mathopen{}\left(w_p\right)\mathclose{}$ to achieve sparsity. The  definitions of MCP \cite{zhang2010nearly} and SCAD \cite{fan2001variable}  with constraints $\gamma\geq 1$ and $\gamma\geq 2$ respectively, are given vt:
\begin{equation}\label{eq5}
 g_{\lambda,\gamma}^{\text{MCP}}\mathopen{}\left(w_p\right)\mathclose{}=
\begin{cases}
\lambda|w_p|-\frac{w_p^2}{2\gamma}, & |w_p| \leq \gamma \lambda 
\\
\frac{\gamma\lambda^2}{2}, & |w_p|> \gamma\lambda   
\end{cases}
\end{equation}
and
\begin{equation}\label{Eq6}
 g_{\lambda,\gamma}^{\text{SCAD}}\mathopen{}\left(w_p\right)\mathclose{}=
\begin{cases}
\lambda|w_p|, & |w_p| \leq \lambda 

\\
-\frac{|w_p|^2-2\gamma\lambda|w_p|+\lambda^2}{2\mathopen{}\left(\gamma-1\right)\mathclose{}}, & \lambda < |w_p| \leq \gamma\lambda 
\\
\frac{\mathopen{}\left(\gamma+1\right)\mathclose{}\lambda^2}{2}, &  |w_p| > \gamma\lambda 
\end{cases}
\end{equation}
These non-convex and non-smooth functions can clearly differentiate between active and non-active coefficients. Additionally, the MCP and SCAD functions are known to be weakly convex for $\rho \geq \frac{1}{\gamma}$ and $\rho \geq \frac{1}{\gamma-1}$ respectively, according to \cite{varma2019vector}. 

Alternatively, by introducing an auxiliary variable $\mathbf{z}$, \eqref{eq3} can be rewritten as:
\begin{alignat}{2}\label{eq4}
&\min_{{\mathbf{w},\mathbf{z}}} &\qquad& \frac{1}{2}\mathopen{}\left(\|\mathbf{z}\|_1 + \mathopen{}\left( 2\tau-1\right)\mathclose{}\mathbf{1}_n^{\text{T}} \mathbf{z}\right)\mathclose{}+ \hspace{0.2mm}n \hspace{1mm} P_{\lambda,\gamma}\mathopen{}\left(\mathbf{w}\right)\mathclose{},\\ \nonumber
&\text{subject to} & & \mathbf{z}+\mathbf{X}\mathbf{w}= \mathbf{y},
\end{alignat}
where $\mathbf{X}=[{\mathbf{x}}_1,\ldots,{\mathbf{x}}_n]^\text{T} \in \mathbb{R}^{n\times P}$ and $\mathbf{y}=[y_1,\cdots,y_n]^\text{T}\in \mathbb{R}^{n}$.
The LLA framework solves this sparse penalized-quantile regression \eqref{eq4} in an iterative procedure by obtaining the sub-gradients of these penalties, i.e., it solves $l_1$ penalized quantile regression in each iteration. However, having a single-loop algorithm can provide more gain as it prevents \emph{secondary convergence iteration} and more accurate solutions. In the next section, we present a smoothing ADMM-based algorithm that alleviates the secondary iteration that appears in the LLA framework.
\subsection{ADMM}
To apply ADMM for solving \eqref{eq4}, one can formulate the associated augmented Lagrangian function as:
\begin{align}\label{eq7}
\begin{split}
\mathcal{L}_{\sigma_{\Psi}}\mathopen{}\left(\mathbf{w},\mathbf{z},\boldsymbol{\Psi}\right)\mathclose{} &= \frac{1}{2}\mathopen{}\left(\|\mathbf{z}\|_1 + \mathopen{}\left( 2\tau-1\right)\mathclose{}\mathbf{1}_{n}^{\text{T}} \mathbf{z}\right)\mathclose{} + n P_{\lambda,\gamma}\mathopen{}\left(\mathbf{w}\right)\mathclose{} \\ 
& \hspace{5mm}+ \boldsymbol{\Psi}^\text{T} \mathopen{}\left(\mathbf{z}+
      \mathbf{X}\mathbf{w}- \mathbf{y}\right)\mathclose{} 
      +   \frac{\sigma_{\Psi}}{2}\mathopen{}\left\|\mathbf{z}+ \mathbf{X}\mathbf{w}- \mathbf{y}\right\|\mathclose{}_2^2,
\end{split}
\end{align}
where $\boldsymbol{\Psi} \in \mathbb{R}^{n}$ denotes the Lagrange multiplier, and $\sigma_{\Psi}$ represents the penalty parameter. ADMM seeks a saddle point for the augmented Lagrangian function through an iterative process \cite{boyd2011distributed}. The $\mathopen{}\left(k+1\right)\mathclose{}$-th iteration of a standard two-block ADMM algorithm can be described as \cite{boyd2011distributed}:
\begin{subequations}\label{eq8}
\begin{equation}\label{eq8.a}
    \mathbf{w}^{\mathopen{}\left(k+1\right)\mathclose{}}=
     \argmin_{\mathbf{w}}  \mathcal{L}_{\sigma_{\Psi}}\mathopen{}\left(\mathbf{w},\mathbf{z}^{\mathopen{}\left(k\right)\mathclose{}},\boldsymbol{\Psi}^{\mathopen{}\left(k\right)\mathclose{}}\right)\mathclose{},
\end{equation}
\begin{equation}\label{eq8.b}
    \mathbf{z}^{\mathopen{}\left(k+1\right)\mathclose{}}=
     \argmin_{\mathbf{z}}  \mathcal{L}_{\sigma_{\Psi}}\mathopen{}\left(\mathbf{w}^{\mathopen{}\left(k+1\right)\mathclose{}},\mathbf{z},\boldsymbol{\Psi}^{\mathopen{}\left(k\right)\mathclose{}}\right)\mathclose{},
\end{equation}
\begin{equation}\label{eq8.c}
    \boldsymbol{\Psi}^{\mathopen{}\left(k+1\right)\mathclose{}}=\boldsymbol{\Psi}^{\mathopen{}\left(k\right)\mathclose{}}+\sigma_{\Psi}\mathopen{}\left(\mathbf{z}^{\mathopen{}\left(k+1\right)\mathclose{}}+ \mathbf{X}\mathbf{w}^{\mathopen{}\left(k+1\right)\mathclose{}}- \mathbf{y}\right)\mathclose{}.
\end{equation}
\end{subequations}
There are two primary issues with applying conventional ADMM to our problem. The first issue is updating the $\mathbf{w}$ variable, which has no closed-form solution when using a general design matrix ${\mathbf{X}}$. Although each fundamental function $g_{\lambda,\gamma}\mathopen{}\left(w_p\right)\mathclose{}$ possesses a closed-form proximal function, acquiring $\mathbf{w}$ does not have a straightforward closed-form solution for a generic design matrix ${\mathbf{X}}$. The $\mathbf{w}$ update step can be performed iteratively using block coordinate descent \cite{boyd2011distributed} which is time-consuming, or through various multi-block ADMM techniques capable of examining each element individually. 

The second issue relates to the absence of Lipschitz differentiability in the objective function. Lipschitz differentiability is a key property in non-convex ADMM algorithms, as it plays a significant role in controlling the changes in the dual parameters based on the primal variables. As a result, Lipschitz differentiability contributes to demonstrating that the value of the augmented Lagrangian decreases over iterations. To the best of our knowledge, existing non-convex ADMM-based approaches \cite{wang2019global, hong2015convergence, yashtini2020convergence, themelis2020douglas} fail to ensure convergence when the objective function lacks these properties. Addressing both of these challenges is crucial for the successful application of ADMM to our problem.
\subsection{Multi-Block ADMM}
One can reformulate \eqref{eq4} as follows:
\begin{alignat}{2}\label{eq9}\noindent \nonumber
&\!\min_{\{\mathbf{w},\mathbf{z}\}} &\qquad&  \frac{1}{2}\mathopen{}\left(\|\mathbf{z}\|_1 + \mathopen{}\left( 2\tau-1\right)\mathclose{}\mathbf{1}_n^{\text{T}} \mathbf{z}\right)\mathclose{}+ \hspace{0.2mm}n \hspace{1mm} \sum_{p=1}^P g_{\lambda,\gamma}\mathopen{}\left(w_p\right)\mathclose{}.\\
&\text{subject to} &      & \mathbf{z}+\mathbf{X}_{:,1}\mathbf{w}_1+\cdots+\mathbf{X}_{:,P}\mathbf{w}_P= \mathbf{y}
\end{alignat}
This formulation allows us to update each $w_p$ individually, as shown in the following:
\begin{multline}\label{eq11}
    w_p^{\mathopen{}\left(k+1\right)\mathclose{}}=
     \argmin_{w_p}  \mathcal{L}_{\sigma_{\Psi}}\mathopen{}\left(\mathbf{w}_{<p}^{\mathopen{}\left(k+1\right)\mathclose{}},w_p,\mathbf{w}_{>p}^{\mathopen{}\left(k\right)\mathclose{}},\mathbf{z}^{\mathopen{}\left(k\right)\mathclose{}},\boldsymbol{\Psi}^{\mathopen{}\left(k\right)\mathclose{}}\right)\mathclose{}, \\ 
     \text{for } p=1,\cdots,P.
\end{multline}
By employing this approach, we introduce multi-block ADMM for our problem. To ensure convergence, it is crucial that the sum of the basic functions, $\sum_{p=1}^P g_{\lambda,\gamma}\mathopen{}\left(w_p\right)\mathclose{}$, satisfies the prox-regularity condition as defined in \cite{poliquin1996prox}. As demonstrated in \cite{wang2019global}, when this condition is met, there will be no convergence issues regarding to updating each element individually. In our case, each function $g_{\lambda,\gamma}\mathopen{}\left(w_p\right)\mathclose{}$, specifically for MCP or SCAD, is the maximum of a set of quadratic functions and according to \cite[Example 2.9]{poliquin1996prox}, such functions are indeed prox-regular.

\subsection{Smoothing Approximation}
Non-smooth objective functions often complicate optimization tasks. One popular approach to tackle non-smoothness is to use smoothing approximation techniques. These techniques approximate the non-smooth functions with smooth functions, which can be more easily optimized based on our proposed ADMM. In this section, we introduce the concept of smoothing approximation, which will be employed to handle the lack of smoothness issue in our proposed ADMM algorithm for sparse-penalized quantile regression.

We begin by providing the definition of a smoothing function, which approximates a given non-smooth function $g$ with a smooth function $\tilde{g}$, facilitating the optimization process.
\begin{customdefinition}{1\cite{chen2012smoothing}} Consider $\tilde{g} : \sigma \subseteq \mathbb{R}^n \times \mathopen{}\left(0, +\infty\right)\mathclose{} \rightarrow \mathbb{R}$, a smoothing function of $g$, where $g : \sigma \subseteq \mathbb{R}^n \rightarrow \mathbb{R}$ exhibits local Lipschitz continuity. The function $\tilde{g}$ possesses the following characteristics:
\begin{enumerate}
\item Continuously differentiable property: For any fixed $\mu > 0$, $\tilde{g}\mathopen{}\left(\cdot, \mu\right)\mathclose{}$ is continuously differentiable in $\mathbb{R}^m$, and for any fixed $\mathbf{x} \in \sigma \subseteq \mathbb{R}^n$, $\tilde{g}\mathopen{}\left(x, \cdot\right)\mathclose{}$ is differentiable in $(0, +\infty]$.
\item Convergence property: For any fixed $\mathbf{x} \in \sigma \subseteq \mathbb{R}^n$, $\lim_{\mu \to 0^+} \tilde{g}\mathopen{}\left(\mathbf{x}, \mu\right)\mathclose{} = g\mathopen{}\left(x\right)\mathclose{}$.
\item Gradient bound property: A positive constant $\kappa_{\tilde{g}} > 0$ exists such that $|\nabla_\mu \tilde{g}\mathopen{}\left(\mathbf{x}, \mu\right)\mathclose{}| \leq \kappa_{\tilde{g}}$ for all $\mu \in \mathopen{}\left(0, +\infty\right)\mathclose{}$ and $\mathbf{x} \in \sigma \subseteq \mathbb{R}^{n}$.
\item Gradient convergence property: $\lim_{\substack{\mathbf{z} \to \mathbf{x} \\ \mu \to 0}} \nabla_z \tilde{g}\mathopen{}\left(\mathbf{z}, \mu\right)\mathclose{} \subseteq \partial g\mathopen{}\left(\mathbf{x}\right)\mathclose{}$.

Additionally, for any fixed $\mathbf{x} \in \mathbb{R}^n$, the smoothing function $\tilde{g}$ satisfies:

\item General convergence property: $\lim_{\substack{\mathbf{z} \to \mathbf{x} \\ \mu \to 0}} \tilde{g}\mathopen{}\left(\mathbf{z}, \mu\right)\mathclose{} = g\mathopen{}\left(\mathbf{x}\right)\mathclose{}$.
\item Lipschitz continuity with respect to $\mu$: There exists a constant $L > 0$ such that $|\tilde{g}\mathopen{}\left(\mathbf{x}, \mu_1\right)\mathclose{} - \tilde{g}\mathopen{}\left(\mathbf{x}, \mu_2\right)\mathclose{}| \leq L|\mu_1 - \mu_2|$.
\item Lipschitz continuity of gradient: $\tilde{g}\mathopen{}\left(\mathbf{x}, \mu\right)\mathclose{}$ is a convex function, and a constant $l > 0$ exists such that $\mathopen{}\left\|\nabla\tilde{g}\mathopen{}\left(\mathbf{x}, \mu\right)\mathclose{} - \nabla \tilde{g}\mathopen{}\left(\mathbf{y}, \mu\right)\mathclose{} \right\|\mathclose{}  \leq l_\mu \mathopen{}\left\|\mathbf{x} - \mathbf{y}\right\|\mathclose{}$ for all $\mathbf{x}, \mathbf{y} \in \sigma \subset \mathbb{R}^n$.
\end{enumerate}
\end{customdefinition}
\begin{lemma}[\cite{bian2013neural}]
    Let $\bar{g}_1,\cdots, \bar{g}_n$ be smoothing functions of ${g}_1,\cdots, {g}_n$, $a_i \geq 0$ and $g_i$ be regular for any $i = 1,2,\cdots,n$. Then, $\sum_{i=1}^{n} a_i \bar{g}_i$ is  a smoothing function of $\sum_{i=1}^{n} a_i {g}_i$.
\end{lemma}
The smoothing function $\tilde{g}$ defined above has several desirable properties, such as smoothness and general convergence property. These properties enable us to perform a more insightful analysis of the optimization problem's convergence, as they simplify the problem's complexity and improve our ability to predict convergence behavior. In the following sections, we will leverage the concept of smoothing approximation to address the challenges posed by the lack of smoothness in the ADMM algorithm for our sparse-penalized quantile regression problem.

\section{Smoothing ADMM for sparse-penalized quantile regression}\label{sec4}
To tackle the challenges above, specifically the absence of Lipschitz differentiability in the objective function and the lack of a closed-form solution for $\mathbf{w}$, we propose a smoothing ADMM-based algorithm with a time-increasing penalty parameter, referred to as SIAD.  First, in each iteration, we approximate the $\|\mathbf{z}\|_1$ using sum of a smooth smoothing approximation function for $|z_i|$ presented in \cite{chen2012smoothing} as $h\mathopen{}\left(\mathbf{z},\mu\right)\mathclose{}=\sum_{i=1}^{n} f\mathopen{}\left(z_i,\mu\right)\mathclose{}$, where: 
\begin{equation}\label{eq12}
 f\mathopen{}\left(z_i,\mu\right)\mathclose{}=
\begin{cases}
|z_i|, &  \mu \leq |z_i|
\\
\frac{z_i^2}{2\mu}+\frac{\mu}{2}. &  |z_i| < \mu 
\end{cases}
\end{equation}
Using the approximation \eqref{eq12} in \eqref{eq7}, the following approximation augmented Lagrangian can be derived:
\begin{multline}\label{eq13}
\bar{\mathcal{L}}_{\sigma_{\Psi},\mu}\mathopen{}\left(\mathbf{w},\mathbf{z},\boldsymbol{\Psi}\right)\mathclose{} = \frac{1}{2}\mathopen{}\left( \sum_{i=1}^n f\mathopen{}\left(z_i,\mu\right)\mathclose{} + \mathopen{}\left( 2\tau-1\right)\mathclose{}\mathbf{1}_{n}^{\text{T}} \mathbf{z}\right)\mathclose{}\\ + n P_{\lambda,\gamma}\mathopen{}\left(\mathbf{w}\right)\mathclose{}  
+ \boldsymbol{\Psi}^\text{T} \mathopen{}\left(\mathbf{z}+
      \mathbf{X}\mathbf{w}- \mathbf{y}\right)\mathclose{} 
      +   \frac{\sigma_{\Psi}}{2}\mathopen{}\left\|\mathbf{z}+ \mathbf{X}\mathbf{w}- \mathbf{y}\right\|\mathclose{}_2^2.
\end{multline}

By having $c>0$ and $\beta>0$ in each iteration, the $\sigma_{\Psi}$, and $\mu^{\mathopen{}\left(k+1\right)\mathclose{}}$ can be updated as:
 \begin{equation}\label{eq14}
     \sigma_{\Psi}^{\mathopen{}\left(k+1\right)\mathclose{}}=c\sqrt{k+1},
 \end{equation}
 and 
  \begin{equation}\label{eq15}
     \mu^{\mathopen{}\left(k+1\right)\mathclose{}}=\frac{\beta}{\sigma_{\Psi}^{\mathopen{}\left(k+1\right)\mathclose{}}}=\frac{\beta}{c\sqrt{k+1}}.
 \end{equation}
Moreover, the update of $\mathbf{w}$ is split into $P$ steps. The $p$-th element of $\mathbf{w}$ is updated in the $p$-th iteration as shown in \eqref{eq11}. After several simplifications, we can see that the update of the $p$-th element of $\mathbf{w}$ is given by
\begin{align}\label{eq16}\nonumber
w_p^{\mathopen{}\left(k+1\right)\mathclose{}}&= \argmin_{w_p}  n g_{\lambda,\gamma}\mathopen{}\left(w_p\right)\mathclose{}+\frac{\sigma_{\Psi}^{\mathopen{}\left(k+1\right)\mathclose{}}\mathopen{}\left\|\mathbf{X}_{:,p}\right\|\mathclose{}_2^2}{2}\mathopen{}\left\|w_p-a_p  \right\|\mathclose{}_2^2, \\
&= \textbf{Prox}_{ g_{\lambda,\gamma}}\mathopen{}\left(a_p;\frac{ n}{\sigma_{\Psi}^{\mathopen{}\left(k+1\right)\mathclose{}}\mathopen{}\left\|\mathbf{X}_{:,p}\right\|\mathclose{}_2^2}\right)\mathclose{},
\end{align}
where
$a_p=\frac{-\mathbf{X}_{:,p}^\text{T}\mathbf{X}_{<p}\mathbf{w}^{\mathopen{}\left(k+1\right)\mathclose{}}_{<p}-\mathbf{X}_{:,p}^\text{T}\mathbf{X}_{>p}\mathbf{w}^{\mathopen{}\left(k\right)\mathclose{}}_{>p}-\mathopen{}\left(\frac{\boldsymbol{\Psi}^{\mathopen{}\left(k\right)\mathclose{}}}{\sigma_{\Psi}^{\mathopen{}\left(k+1\right)\mathclose{}}}+\mathbf{y}-\mathbf{z}^{\mathopen{}\left(k\right)\mathclose{}}\right)\mathclose{}^\text{T}\mathbf{X}_p}{\mathopen{}\left\|\mathbf{X}_{:,p}\right\|\mathclose{}_2^2}.$
Both MCP and SCAD admit closed-form solutions of the proximal operator \cite{huang2012selective}. 
Next, the update of $\mathbf{z}$ can be formulated as:
\begin{align}\label{eq17}
    \mathbf{z}^{\mathopen{}\left(k+1\right)\mathclose{}}= \argmin_{\mathbf{z}}\frac{1}{2}\mathopen{}\left( \sum_{i=1}^n f\mathopen{}\left(z_i,\mu^{\mathopen{}\left(k+1\right)\mathclose{}}\right)\mathclose{} + \mathopen{}\left( 2\tau-1\right)\mathclose{}\mathbf{1}_{n}^{\text{T}} \mathbf{z}\right)\mathclose{} 
     \nonumber\\ 
   +\mathopen{}\left(\boldsymbol{\Psi}^{\mathopen{}\left(k\right)\mathclose{}}\right)\mathclose{}^\text{T} \mathbf{z}+   \frac{\sigma_{\Psi}^{\mathopen{}\left(k+1\right)\mathclose{}}}{2}\mathopen{}\left\|\mathbf{z}+\mathbf{X}\mathbf{w}^{\mathopen{}\left(k+1\right)\mathclose{}}- \mathbf{y}\right\|\mathclose{}_2^2.
\end{align}
It can be shown that the update step of $\mathbf{z}$ in ADMM has a closed-form solution. By merging $\mathopen{}\left(\tau-\frac{1}{2}\right)\mathclose{}\mathbf{1}_{n}^{\text{T}} \mathbf{z}$, $\mathopen{}\left(\boldsymbol{\Psi}^{\mathopen{}\left(k\right)\mathclose{}}\right)\mathclose{}^\text{T} \mathbf{z}$, and $\mathopen{}\left\|\mathbf{z}+\mathbf{X}\mathbf{w}^{\mathopen{}\left(k+1\right)\mathclose{}}- \mathbf{y}\right\|\mathclose{}_2^2$ together, a component-wise solution can be obtained as
\begin{align}\label{eq18}
    \mathbf{z}^{\mathopen{}\left(k+1\right)\mathclose{}}
       &=\argmin_{\mathbf{z}} \sum_{i=1}^n f\mathopen{}\left(z_i,\mu^{\mathopen{}\left(k+1\right)\mathclose{}}\right)\mathclose{}
      +\frac{1}{2}\mathopen{}\left\|\mathbf{z}-\boldsymbol{\alpha}\right\|\mathclose{}_2^2
      \nonumber \\
&=\textbf{Prox}_{f\mathopen{}\left(\cdot,\mu^{\mathopen{}\left(k+1\right)\mathclose{}}\right)\mathclose{}}\mathopen{}\left(\alpha_i;\frac{\sigma_{\Psi}^{\mathopen{}\left(k+1\right)\mathclose{}}}{2}\right)\mathclose{}_{i=1}^{n},
\end{align}
where
$\boldsymbol{\alpha}=\mathopen{}\left(\mathbf{y}-\mathbf{X} \mathbf{w}^{\mathopen{}\left(k+1\right)\mathclose{}}\right)\mathclose{}-\frac{\boldsymbol{\Psi}^{\mathopen{}\left(k\right)\mathclose{}}+\mathopen{}\left(\tau-\frac{1}{2}\right)\mathclose{}\mathbf{1}_{n}}{\sigma_{\Psi}^{\mathopen{}\left(k+1\right)\mathclose{}}}$, and \begin{equation}\label{eq19}
\textbf{Prox}_{f\mathopen{}\left(\cdot,{\mu}\right)\mathclose{}}\mathopen{}\left(x;\rho\right)\mathclose{}=
\begin{cases}
x-\rho, & x \geq  \rho + \mu
\\
\frac{z}{1+\frac{\rho}{\mu}}, & -\rho - \mu < x < \rho + \mu 
\\
x+\rho. &  x < -\rho - \mu 
\end{cases}
\end{equation} Finally, the update of dual variable $\boldsymbol{\Psi}$ is given by
 \begin{equation}\label{eq20}
     \boldsymbol{\Psi}^{\mathopen{}\left(k+1\right)\mathclose{}}=\boldsymbol{\Psi}^{\mathopen{}\left(k\right)\mathclose{}}+\sigma_{\Psi}^{\mathopen{}\left(k+1\right)\mathclose{}}\mathopen{}\left(\mathbf{z}^{\mathopen{}\left(k+1\right)\mathclose{}}+ \mathbf{X}\mathbf{w}^{\mathopen{}\left(k+1\right)\mathclose{}}- \mathbf{y}\right)\mathclose{}.
 \end{equation}

The proposed ADMM-based method for solving the sparse-penalized quantile regression is summarized in Algorithm \ref{alg:1}.
\begin{algorithm}[t]
 \caption{Smoothing ADMM with time-increasing penalty parameter (SIAD) for sparse-penalized quantile regression}
 \label{alg:1}
\SetAlgoLined
 Initialize $\mathbf{w}^{\mathopen{}\left(0\right)\mathclose{}}$, $\mathbf{z}^{\mathopen{}\left(0\right)\mathclose{}}$, $\boldsymbol{\beta}^{\mathopen{}\left(0\right)\mathclose{}}$, $\boldsymbol{\Psi}^{\mathopen{}\left(0\right)\mathclose{}}$ to zero vectors and $\beta>1$\;
 \Repeat{the convergence
criterion in \eqref{eq21} is met}{
 Update $\sigma_{\Psi}^{\mathopen{}\left(k+1\right)\mathclose{}}$ by \eqref{eq14}\;
 Update $\mu^{\mathopen{}\left(k+1\right)\mathclose{}}$ by \eqref{eq15}\;
    \For{$p=1,\ldots,P$}{
     Update $w_p^{\mathopen{}\left(k+1\right)\mathclose{}}$ by \eqref{eq16}\;
    }  
  Update $\mathbf{z}^{\mathopen{}\left(k+1\right)\mathclose{}}$ by \eqref{eq18}\;
   Update $\boldsymbol{\Psi}^{\mathopen{}\left(k+1\right)\mathclose{}}$ by \eqref{eq20}\;

 }
\end{algorithm}
It is worth mentioning that the stopping criterion in \cite{boyd2011distributed} can be adapted to our problem as:
\begin{subequations}\label{eq21}
\begin{multline}\label{eq21a}
    \|\mathbf{z}^{\mathopen{}\left(k+1\right)\mathclose{}}+ \mathbf{X}\mathbf{w}^{\mathopen{}\left(k+1\right)\mathclose{}}- \mathbf{y}\|_2\leq \sqrt{n}\epsilon_1 \\ 
    +\epsilon_2 \max\{\|\mathbf{X}\mathbf{w}^{\mathopen{}\left(k+1\right)\mathclose{}}\|_2,\|\mathbf{z}^{\mathopen{}\left(k+1\right)\mathclose{}}\|_2,\|\mathbf{y}^{\mathopen{}\left(k+1\right)\mathclose{}}\|_2\},
\end{multline} \label{eq21b}    
\begin{multline}
     \sigma_{\Psi}^{\mathopen{}\left(k+1\right)\mathclose{}}P\max_p \mathopen{}\left\|\mathbf{X}_{:,p}\right\|\mathclose{}_2^2 \|\mathbf{w}^{\mathopen{}\left(k+1\right)\mathclose{}}-\mathbf{w}^{\mathopen{}\left(k\right)\mathclose{}}\|_2 \\ 
  +\sigma_{\Psi}^{\mathopen{}\left(k+1\right)\mathclose{}}\|\mathbf{X}^{\text{T}}\mathopen{}\left(\mathbf{z}^{\mathopen{}\left(k+1\right)\mathclose{}}-\mathbf{z}^{\mathopen{}\left(k\right)\mathclose{}}\right)\mathclose{}\|_2\leq \sqrt{P}\epsilon_1 +\epsilon_2 \|\mathbf{X}^{\text{T}}\boldsymbol{\Psi}^{\mathopen{}\left(k+1\right)\mathclose{}}\|_2,
  \end{multline} 
\end{subequations}
where a typical choice for $\epsilon_1$ and $\epsilon_2$ is $10^{-3}$. Alternatively, the algorithm can be terminated when the number of iterations exceeds a certain number. 
\section{Convergence proof}\label{sec5}
Establishing the convergence of the proposed ADMM algorithm is contingent upon validating four crucial conditions: boundedness, sufficient descent, subgradient bound, and continuity, as highlighted in \cite[Theorem 2.9]{attouch2013convergence}. To this end, we start by demonstrating the boundedness of the augmented Lagrangian, formalized in Lemma \ref{lem2}.
\begin{lemma}[Boundedness Property]\label{lem2}
The augmented Lagrangian $\bar{\mathcal{L}}_{\sigma_{\Psi},\mu}\mathopen{}\left(\mathbf{w},\mathbf{z},\boldsymbol{\Psi}\right)\mathclose{}$ is lower bounded.
\end{lemma}
\begin{proof}
Consider the augmented Lagrangian given by \eqref{eq13}. Utilizing the optimality conditions of the ADMM algorithm in the $\mathbf{z}$-update step, it follows that $\boldsymbol{\Psi}^{\mathopen{}\left(k+1\right)\mathclose{}} = \frac{1}{2} \nabla \mathopen{}\left( \sum_{i=1}^n f\mathopen{}\left(z_i^{\mathopen{}\left(k+1\right)\mathclose{}},\mu^{\mathopen{}\left(k+1\right)\mathclose{}}\right)\mathclose{} + \mathopen{}\left( 2\tau-1\right)\mathclose{}\mathbf{1}_{n}^{\text{T}} \mathbf{z}^{\mathopen{}\left(k+1\right)\mathclose{}}\right)\mathclose{} \in  [\tau -1 ,\tau]^n$. Since each $\boldsymbol{\Psi}^{(k+1)}$ is bounded, the term $\mathopen{}\left(\boldsymbol{\Psi}^{\mathopen{}\left(k+1\right)\mathclose{}}\right)\mathclose{}^\text{T} \mathopen{}\left(\mathbf{z}^{\mathopen{}\left(k+1\right)\mathclose{}}+
      \mathbf{X}\mathbf{w}^{\mathopen{}\left(k+1\right)\mathclose{}}- \mathbf{y}\right)\mathclose{}=\mathopen{}\left(\boldsymbol{\Psi}^{\mathopen{}\left(k+1\right)\mathclose{}}\right)\mathclose{}^\text{T} \frac{\mathopen{}\left(\boldsymbol{\Psi}^{\mathopen{}\left(k+1\right)\mathclose{}}-\boldsymbol{\Psi}^{\mathopen{}\left(k\right)\mathclose{}}\right)\mathclose{}}{\sigma_{\Psi}^{\mathopen{}\left(k+1\right)\mathclose{}}}$ is also lower bounded. Given that both $h(\mathbf{z},\mu)$ and $P_{\lambda,\gamma}(\mathbf{w})$ are non-negative functions in the augmented Lagrangian, and that the term $\mathopen{}\left(\boldsymbol{\Psi}^{\mathopen{}\left(k+1\right)\mathclose{}}\right)\mathclose{}^\text{T} \mathopen{}\left(\mathbf{z}^{\mathopen{}\left(k+1\right)\mathclose{}}+
      \mathbf{X}\mathbf{w}^{\mathopen{}\left(k+1\right)\mathclose{}}- \mathbf{y}\right)\mathclose{}$ could be negative but is lower bounded, we conclude that the augmented Lagrangian $\bar{\mathcal{L}}_{\sigma_{\Psi},\mu}(\mathbf{w}, \mathbf{z}, \boldsymbol{\Psi})$ is lower bounded.
\end{proof}
The implication of Lemma \ref{lem2} is that the dual variable values are restricted, and the augmented Lagrangian is lower-bounded. Further, from the optimality condition, we can infer that the parameters are also upper-bounded.

As we go forward, it is essential to understand the impact of each update step on the augmented Lagrangian function. We begin by considering the updates to the penalty parameter $\sigma_{\Psi}$, and smoothing parameter $\mu$ in a more general setting than the algorithm \ref{alg:1}.

\begin{lemma}\label{lem3} Assuming that $\sigma_{\Psi}$, $\mu$ are updated in each iteration, following the update rules $\sigma_{\Psi}^{(k+1)}\geq \sigma_{\Psi}^{(k)}$ and $\mu^{(k+1)}\leq \mu^{(k)}$, the change in the augmented Lagrangian function in the $\mu$ update step is non-positive, and in the $\sigma_{\Psi}$ update step can be expressed as:
    \begin{multline}\label{eq212}
\bar{\mathcal{L}}_{\sigma_{\Psi}^{\mathopen{}\left(k+1\right)\mathclose{}},\mu^{\mathopen{}\left(k\right)\mathclose{}}}\mathopen{}\left(\mathbf{w}^{\mathopen{}\left(k\right)\mathclose{}},\mathbf{z}^{\mathopen{}\left(k\right)\mathclose{}},\boldsymbol{\Psi}^{\mathopen{}\left(k\right)\mathclose{}}\right)\mathclose{} \\ -\bar{\mathcal{L}}_{\sigma_{\Psi}^{\mathopen{}\left(k\right)\mathclose{}},\mu^{\mathopen{}\left(k\right)\mathclose{}}}\mathopen{}\left(\mathbf{w}^{\mathopen{}\left(k\right)\mathclose{}},\mathbf{z}^{\mathopen{}\left(k\right)\mathclose{}},\boldsymbol{\Psi}^{\mathopen{}\left(k\right)\mathclose{}}\right)\mathclose{}\\=\frac{\sigma_{\Psi}^{\mathopen{}\left(k+1\right)\mathclose{}}-\sigma_{\Psi}^{\mathopen{}\left(k\right)\mathclose{}}}{\mathopen{}\left(\sigma_{\Psi}^{\mathopen{}\left(k\right)\mathclose{}}\right)\mathclose{}^2}\mathopen{}\left\|\boldsymbol{\Psi}^{\mathopen{}\left(k\right)\mathclose{}}-\boldsymbol{\Psi}^{\mathopen{}\left(k-1\right)\mathclose{}}\right\|\mathclose{}_2^2.
    \end{multline}
\end{lemma}
\begin{proof}
The variation between the augmented Lagrangian function values at iteration $k+1$ after updating $\sigma_{\Psi}$ is given by:
\begin{subequations}\label{eq22}
    \begin{multline}\label{eq22a}
\bar{\mathcal{L}}_{\sigma_{\Psi}^{\mathopen{}\left(k+1\right)\mathclose{}},\mu^{\mathopen{}\left(k\right)\mathclose{}}}\mathopen{}\left(\mathbf{w}^{\mathopen{}\left(k\right)\mathclose{}},\mathbf{z}^{\mathopen{}\left(k\right)\mathclose{}},\boldsymbol{\Psi}^{\mathopen{}\left(k\right)\mathclose{}}\right)\mathclose{}\\-\bar{\mathcal{L}}_{\sigma_{\Psi}^{\mathopen{}\left(k\right)\mathclose{}},\mu^{\mathopen{}\left(k\right)\mathclose{}}}\mathopen{}\left(\mathbf{w}^{\mathopen{}\left(k\right)\mathclose{}},\mathbf{z}^{\mathopen{}\left(k\right)\mathclose{}},\boldsymbol{\Psi}^{\mathopen{}\left(k\right)\mathclose{}}\right)\mathclose{} \qquad \quad
    \end{multline}
    \begin{equation}\label{eq22b}
      \qquad \qquad=   \mathopen{}\left(\sigma_{\Psi}^{\mathopen{}\left(k+1\right)\mathclose{}}-\sigma_{\Psi}^{\mathopen{}\left(k\right)\mathclose{}}\right)\mathclose{} \mathopen{}\left\|\mathbf{z}^{\mathopen{}\left(k\right)\mathclose{}}+ \mathbf{X}\mathbf{w}^{\mathopen{}\left(k\right)\mathclose{}}- \mathbf{y}\right\|\mathclose{}_2^2.
    \end{equation}
\end{subequations}
Considering that $\mathopen{}\left\|\mathbf{z}^{\mathopen{}\left(k\right)\mathclose{}}+ \mathbf{X}\mathbf{w}^{\mathopen{}\left(k\right)\mathclose{}}- \mathbf{y}\right\|\mathclose{}_2^2= \mathopen{}\left({\frac{1}{\sigma_{\Psi}^{\mathopen{}\left(k\right)\mathclose{}}}}\right)\mathclose{}^2\mathopen{}\left\|\boldsymbol{\Psi}^{\mathopen{}\left(k\right)\mathclose{}}-\boldsymbol{\Psi}^{k-1}\right\|\mathclose{}_2^2$, we substitute this expression into \eqref{eq22b} to  obtain \eqref{eq212}.

Moreover, by having $\mu^{\mathopen{}\left(k+1\right)\mathclose{}} \leq \mu^{\mathopen{}\left(k\right)\mathclose{}}$, we can see $h_{\mu^{\mathopen{}\left(k+1\right)\mathclose{}}}\mathopen{}\left(\mathbf{z}\right)\mathclose{}\leq h_{\mu^{\mathopen{}\left(k\right)\mathclose{}}}\mathopen{}\left(\mathbf{z}\right)\mathclose{}$, which arises from the monotonicity of the function $h_\mu\mathopen{}\left(\mathbf{z}\right)\mathclose{}$ with respect to the parameter $\mu$. Consequently, by updating $\mu$ according to \eqref{eq15}, the value of the approximation function either reduces or stays the same, mirroring the behavior of the augmented Lagrangian as:
\begin{subequations}\label{eq211}
   \begin{multline}\label{eq211a}
\bar{\mathcal{L}}_{\sigma_{\Psi}^{\mathopen{}\left(k+1\right)\mathclose{}},\mu^{\mathopen{}\left(k+1\right)\mathclose{}}}\mathopen{}\left(\mathbf{w}^{\mathopen{}\left(k\right)\mathclose{}},\mathbf{z}^{\mathopen{}\left(k\right)\mathclose{}},\boldsymbol{\Psi}^{\mathopen{}\left(k\right)\mathclose{}}\right)\mathclose{}  \\ -\bar{\mathcal{L}}_{\sigma_{\Psi}^{\mathopen{}\left(k+1\right)\mathclose{}},\mu^{\mathopen{}\left(k\right)\mathclose{}}}\mathopen{}\left(\mathbf{w}^{\mathopen{}\left(k\right)\mathclose{}},\mathbf{z}^{\mathopen{}\left(k\right)\mathclose{}},\boldsymbol{\Psi}^{\mathopen{}\left(k\right)\mathclose{}}\right)\mathclose{} \quad \quad
\end{multline}
\begin{equation}
   \qquad \qquad \qquad \qquad \qquad \leq 0.
\end{equation}
\end{subequations}
\end{proof}
Lemma \ref{lem3} suggests a controlled increase in the value of the augmented Lagrangian due to the incremental update of the penalty parameter $\sigma_{\Psi}$. Moreover, it asserts that a decremental update in $\mu$  always results in a non-increase change for the augmented Lagrangian. Next, we take into account the implications of $\mathbf{w}$, $\mathbf{z}$, and $\boldsymbol{\Psi}$ updates.
\begin{lemma}\label{lem4}
Let $\xi^{\mathopen{}\left(k+1\right)\mathclose{}} = \frac{\sigma_{\Psi}^{\mathopen{}\left(k+1\right)\mathclose{}} \min_p |\mathbf{X}{:,p}|^2_2}{2n} - \rho$, following the algorithm \ref{alg:1} we have the below results:
\begin{enumerate}
    \item The decrease in the update step of $\mathbf{w}$ can be upper bounded as follows:
    \begin{subequations}\label{eq30}
        \begin{multline}\label{eq30a}
\bar{\mathcal{L}}_{\sigma_{\Psi}^{\mathopen{}\left(k+1\right)\mathclose{}},\mu^{\mathopen{}\left(k+1\right)\mathclose{}}}\mathopen{}\left(\mathbf{w}^{\mathopen{}\left(k+1\right)\mathclose{}},\mathbf{z}^{\mathopen{}\left(k\right)\mathclose{}},\boldsymbol{\Psi}^{\mathopen{}\left(k\right)\mathclose{}}\right)\mathclose{}\\ -\bar{\mathcal{L}}_{\sigma_{\Psi}^{\mathopen{}\left(k+1\right)\mathclose{}},\mu^{\mathopen{}\left(k+1\right)\mathclose{}}}\mathopen{}\left(\mathbf{w}^{\mathopen{}\left(k\right)\mathclose{}},\mathbf{z}^{\mathopen{}\left(k\right)\mathclose{}},\boldsymbol{\Psi}^{\mathopen{}\left(k\right)\mathclose{}}\right)\mathclose{} \quad 
    \end{multline}
    \begin{equation}
     \quad \qquad \qquad  \leq -{\xi^{\mathopen{}\left(k+1\right)\mathclose{}}}\mathopen{}\left\|\mathbf{w}^{\mathopen{}\left(k+1\right)\mathclose{}}-\mathbf{w}^{\mathopen{}\left(k\right)\mathclose{}}\right\|\mathclose{}_2^2. 
    \end{equation}
     \end{subequations}
    \item The decrease in the $\mathbf{z}$ update step can be upper bounded as 
    \begin{subequations}\label{eq261}
        \begin{multline}\label{eq261a}
\bar{\mathcal{L}}_{\sigma_{\Psi}^{\mathopen{}\left(k+1\right)\mathclose{}},\mu^{\mathopen{}\left(k+1\right)\mathclose{}}}\mathopen{}\left(\mathbf{w}^{\mathopen{}\left(k+1\right)\mathclose{}},\mathbf{z}^{\mathopen{}\left(k+1\right)\mathclose{}},\boldsymbol{\Psi}^{\mathopen{}\left(k\right)\mathclose{}}\right)\mathclose{}\\-\bar{\mathcal{L}}_{\sigma_{\Psi}^{\mathopen{}\left(k+1\right)\mathclose{}},\mu^{\mathopen{}\left(k+1\right)\mathclose{}}}\mathopen{}\left(\mathbf{w}^{\mathopen{}\left(k+1\right)\mathclose{}},\mathbf{z}^{\mathopen{}\left(k\right)\mathclose{}},\boldsymbol{\Psi}^{\mathopen{}\left(k\right)\mathclose{}}\right)\mathclose{} \quad 
    \end{multline}
    \begin{equation}
     \qquad \qquad \qquad  \leq \frac{\sigma_{\Psi}^{\mathopen{}\left(k+1\right)\mathclose{}}}{2}\mathopen{}\left\|\mathbf{z}^{\mathopen{}\left(k\right)\mathclose{}}-\mathbf{z}^{\mathopen{}\left(k-1\right)\mathclose{}}\right\|\mathclose{}_2^2. 
    \end{equation}
     \end{subequations}
     
    \item The increase in the dual update step can be expressed as:
    \begin{multline}\label{eq271}
    \bar{\mathcal{L}}_{\sigma_{\Psi}^{\mathopen{}\left(k+1\right)\mathclose{}},\mu^{\mathopen{}\left(k+1\right)\mathclose{}}}\mathopen{}\left(\mathbf{w}^{\mathopen{}\left(k+1\right)\mathclose{}},\mathbf{z}^{\mathopen{}\left(k+1\right)\mathclose{}},\boldsymbol{\Psi}^{\mathopen{}\left(k+1\right)\mathclose{}}\right)\mathclose{} \\ -  \bar{\mathcal{L}}_{\sigma_{\Psi}^{\mathopen{}\left(k+1\right)\mathclose{}},\mu^{\mathopen{}\left(k+1\right)\mathclose{}}}\mathopen{}\left(\mathbf{w}^{\mathopen{}\left(k+1\right)\mathclose{}},\mathbf{z}^{\mathopen{}\left(k+1\right)\mathclose{}},\boldsymbol{\Psi}^{\mathopen{}\left(k\right)\mathclose{}}\right)\mathclose{}\\ =\frac{1}{\sigma_{\Psi}^{\mathopen{}\left(k+1\right)\mathclose{}}}\mathopen{}\left\|\boldsymbol{\Psi}^{\mathopen{}\left(k+1\right)\mathclose{}}-\boldsymbol{\Psi}^{\mathopen{}\left(k\right)\mathclose{}}\right\|\mathclose{}_2^2.
    \end{multline}
\end{enumerate}
\end{lemma}
\begin{proof}
To quantify the effect of updating the primal variable \( \mathbf{w} \), the update of each \( {w}_p\) as described in \eqref{eq16} is considered first. Due to the weak convexity of MCP and SCAD, with parameters \(\rho = \frac{1}{\gamma}\) and \(\rho = \frac{1}{\gamma - 1}\), the following inequality holds:
\begin{multline}
\bar{\mathcal{L}}_{\sigma_{\Psi}^{\mathopen{}\left(k+1\right)\mathclose{}},\mu^{\mathopen{}\left(k+1\right)\mathclose{}}}\mathopen{}\left(\mathbf{w}_{<p}^{\mathopen{}\left(k+1\right)\mathclose{}},w_{p}^{\mathopen{}\left(k+1\right)\mathclose{}},\mathbf{w}_{>p}^{\mathopen{}\left(k\right)\mathclose{}},\mathbf{z}^{\mathopen{}\left(k\right)\mathclose{}},\boldsymbol{\Psi}^{\mathopen{}\left(k\right)\mathclose{}}\right)\mathclose{} - \\ \bar{\mathcal{L}}_{\sigma_{\Psi}^{\mathopen{}\left(k+1\right)\mathclose{}},\mu^{\mathopen{}\left(k+1\right)\mathclose{}}}\mathopen{}\left(\mathbf{w}_{<p}^{\mathopen{}\left(k+1\right)\mathclose{}},w_{p}^{\mathopen{}\left(k\right)\mathclose{}},\mathbf{w}_{>p}^{\mathopen{}\left(k\right)\mathclose{}},\mathbf{z}^{\mathopen{}\left(k\right)\mathclose{}},\boldsymbol{\Psi}^{\mathopen{}\left(k\right)\mathclose{}}\right)\mathclose{} \\
  \leq \mathopen{}\left(\frac{\sigma_{\Psi}^{\mathopen{}\left(k+1\right)\mathclose{}}\mathopen{}\left\| \mathbf{X}_{:,p}\right\|\mathclose{}_2^2}{2n} -\rho\right)\mathclose{}\mathopen{}\left\|w_p^{\mathopen{}\left(k+1\right)\mathclose{}}-w_p^{\mathopen{}\left(k\right)\mathclose{}}\right\|\mathclose{}_2^2,
\end{multline}
where the inequality arises from the definition of weak convexity. Thus, combining the inequalities for each $w_p$, leads to:
\begin{multline}
\bar{\mathcal{L}}_{\sigma_{\Psi}^{\mathopen{}\left(k+1\right)\mathclose{}},\mu^{\mathopen{}\left(k+1\right)\mathclose{}}} \mathopen{}\left(\mathbf{w}^{\mathopen{}\left(k+1\right)\mathclose{}},\mathbf{z}^{\mathopen{}\left(k\right)\mathclose{}},\boldsymbol{\Psi}^{\mathopen{}\left(k\right)\mathclose{}}\right)\mathclose{} \\ -\bar{\mathcal{L}}_{\sigma_{\Psi}^{\mathopen{}\left(k+1\right)\mathclose{}},\mu^{\mathopen{}\left(k+1\right)\mathclose{}}}\mathopen{}\left(\mathbf{w}^{\mathopen{}\left(k\right)\mathclose{}},\mathbf{z}^{\mathopen{}\left(k\right)\mathclose{}},\boldsymbol{\Psi}^{\mathopen{}\left(k\right)\mathclose{}}\right)\mathclose{} \\ \leq \mathopen{}\left(\frac{\sigma_{\Psi}^{\mathopen{}\left(k+1\right)\mathclose{}} \min_p \mathopen{}\left\|\mathbf{X}_{:,p}\right\|\mathclose{}_2^2}{2n}-\rho\right)\mathclose{}\mathopen{}\left\|\mathbf{w}^{\mathopen{}\left(k+1\right)\mathclose{}}-\mathbf{w}^{\mathopen{}\left(k\right)\mathclose{}}\right\|\mathclose{}_2^2,
\end{multline}
which confirms the result in \eqref{eq30}.

Next, the effect of changes in the primal variable \( \mathbf{z} \) over the augmented Lagrangian are examined. Due to the convexity of each \( h_\mu(.) \), the augmented Lagrangian is \(\sigma_{\Psi}^{(k+1)}\)-strongly convex with respect to \( \mathbf{z} \). Thus, updating $\mathbf{z}$ based on the optimality of the update step contributes to a decrease in the augmented Lagrangian by at least $\frac{\sigma_{\Psi}^{\mathopen{}\left(k+1\right)\mathclose{}}}{2}\mathopen{}\left\|\mathbf{z}^{\mathopen{}\left(k\right)\mathclose{}}-\mathbf{z}^{\mathopen{}\left(k-1\right)\mathclose{}}\right\|\mathclose{}_2^2$, as established in  \eqref{eq261}.

    Lastly, to analyze the impact of the dual update,  the difference in the Lagrangian function after and before this step in the \(k+1\) iteration is calculated:
    \begin{subequations}\label{eq29}
        \begin{multline}\label{eq29a} \bar{\mathcal{L}}_{\sigma_{\Psi}^{\mathopen{}\left(k+1\right)\mathclose{}},\mu^{\mathopen{}\left(k+1\right)\mathclose{}}}\mathopen{}\left(\mathbf{w}^{\mathopen{}\left(k+1\right)\mathclose{}},\mathbf{z}^{\mathopen{}\left(k+1\right)\mathclose{}},\boldsymbol{\Psi}^{\mathopen{}\left(k+1\right)\mathclose{}}\right)\mathclose{}\\-\bar{\mathcal{L}}_{\sigma_{\Psi}^{\mathopen{}\left(k+1\right)\mathclose{}},\mu^{\mathopen{}\left(k+1\right)\mathclose{}}}\mathopen{}\left(\mathbf{w}^{\mathopen{}\left(k+1\right)\mathclose{}},\mathbf{z}^{\mathopen{}\left(k+1\right)\mathclose{}},\boldsymbol{\Psi}^{\mathopen{}\left(k\right)\mathclose{}}\right)\mathclose{}
\end{multline}
\begin{equation}\label{eq29b}
    =\langle\boldsymbol{\Psi}^{\mathopen{}\left(k+1\right)\mathclose{}}-\boldsymbol{\Psi}^{\mathopen{}\left(k\right)\mathclose{}}, \mathbf{z}^{\mathopen{}\left(k+1\right)\mathclose{}}+ \mathbf{X}\mathbf{w}^{\mathopen{}\left(k+1\right)\mathclose{}}- \mathbf{y} \rangle.
\end{equation}
    \end{subequations}
    Manipulating \eqref{eq29b} to yield $ \langle\boldsymbol{\Psi}^{\mathopen{}\left(k+1\right)\mathclose{}}-\boldsymbol{\Psi}^{\mathopen{}\left(k\right)\mathclose{}},\frac{\boldsymbol{\Psi}^{\mathopen{}\left(k+1\right)\mathclose{}}-\boldsymbol{\Psi}^{\mathopen{}\left(k\right)\mathclose{}}}{\sigma_{\Psi}^{\mathopen{}\left(k+1\right)\mathclose{}}} \rangle=\frac{1}{\sigma_{\Psi}^{\mathopen{}\left(k+1\right)\mathclose{}}}\mathopen{}\left\|\boldsymbol{\Psi}^{\mathopen{}\left(k+1\right)\mathclose{}}-\boldsymbol{\Psi}^{\mathopen{}\left(k\right)\mathclose{}}\right\|\mathclose{}_2^2$, confirms the result in  \eqref{eq271}.
\end{proof}
Based on Lemma \ref{lem4}, the dual update step has the potential to be an increasing step. Similarly, Lemma \ref{lem3} suggests that the $\sigma_{\Psi}$ update step can also exhibit an increasing nature. These steps clearly depend on the variations in the dual variable. Subsequently, we aim to restrict these variations in the dual variable using the primal variables as a foundation. This restriction is intended to demonstrate that the entire iteration process is, in fact, decreasing.

\begin{lemma}\label{lem5}
Function $h\mathopen{}\left(\mathbf{z},\mu\right)\mathclose{}=\frac{1}{2}\mathopen{}\left( \sum_{i=1}^n f\mathopen{}\left(z_i,\mu\right)\mathclose{} + \mathopen{}\left( 2\tau-1\right)\mathclose{}\mathbf{1}_{n}^{\text{T}} \mathbf{z}\right)\mathclose{}$ is $L=\frac{1}{2\mu}$ smooth.
\end{lemma}
\begin{proof}
Let us consider the gradient of the function $f\mathopen{}\left(z_p,\mu\right)\mathclose{}$ as defined below: 
    \begin{equation}\label{eq.grad}
\nabla f\mathopen{}\left(z_p,\mu\right)\mathclose{}=
\begin{cases}
\text{Sign}\mathopen{}\left(z_p\right)\mathclose{}, & \mu \leq |z_p|
\\
\frac{z_p}{\mu}. &  |z_p| < \mu 
\end{cases}
\end{equation}
Based on this definition, we obtain $\mathopen{}\left\|\nabla f\mathopen{}\left(x_i,\mu\right)\mathclose{}- \nabla f\mathopen{}\left(y_i,\mu\right)\mathclose{}\right\|\mathclose{}_2=\frac{1}{\mu}\mathopen{}\left\|x_i-y_i\right\|\mathclose{}$. As the function $h\mathopen{}\left(\mathbf{z},\mu\right)\mathclose{}$ is composed of $L=\frac{1}{2\mu}$ smooth function $\frac{1}{2}\sum_{i=1}^n f\mathopen{}\left(z_i,\mu\right)\mathclose{}$ and a linear term $\frac{1}{2}\mathopen{}\left(\mathopen{}\left( 2\tau-1\right)\mathclose{}\mathbf{1}_{n}^{\text{T}} \mathbf{z}\right)\mathclose{}$, we can conclude that $h\mathopen{}\left(\mathbf{z},\mu\right)\mathclose{}$ is $L=\frac{1}{2\mu}$ smooth.
\end{proof}
 \begin{lemma}\label{lem6} By progressively updating $\mu$ in such a way that ${\mu^{\mathopen{}\left(k+1\right)\mathclose{}}} < {\mu^{\mathopen{}\left(k\right)\mathclose{}}}$, we can bound the variation in the derivative of the approximate function  $h\mathopen{}\left(\cdot,{\mu^{\mathopen{}\left(k\right)\mathclose{}}}\right)\mathclose{}$ as:
 \begin{equation}\label{eq222}
  \mathopen{}\left\|\nabla  h\mathopen{}\left(\mathbf{z},{\mu^{\mathopen{}\left(k\right)\mathclose{}}}\right)\mathclose{} -\nabla  h\mathopen{}\left(\mathbf{z},{\mu^{\mathopen{}\left(k+1\right)\mathclose{}}}\right)\mathclose{}\right\|\mathclose{} \leq \frac{n}{2} \mathopen{}\left(\frac{\mu^{\mathopen{}\left(k\right)\mathclose{}}-\mu^{\mathopen{}\left(k+1\right)\mathclose{}}}{\mu^{\mathopen{}\left(k\right)\mathclose{}}}\right)\mathclose{}. 
  \end{equation}
 \end{lemma}
\begin{proof}
Starting from $\mathopen{}\left\|\nabla f\mathopen{}\left(z,{\mu^{\mathopen{}\left(k\right)\mathclose{}}}\right)\mathclose{} - \nabla f\mathopen{}\left(z,{\mu^{\mathopen{}\left(k+1\right)\mathclose{}}}\right)\mathclose{}\right\|\mathclose{}$ and applying the form of the gradient from \eqref{eq.grad}, three scenarios arise:
\begin{itemize}
    \item \textit{Case 1}: If $\mu^{\mathopen{}\left(k\right)\mathclose{}}<|z|$, the difference between the gradients is zero.
    \item \textit{Case 2}: When  $\mu^{\left(k+1\right)} \leq |z| \leq \mu^{\left(k\right)}$, then the difference between the gradients is 
    \begin{align*}
    \left| \text{Sign}(z) - \frac{z}{\mu^{(k)}} \right| &= \left| 1 - \frac{|z|}{\mu^{(k)}} \right| \\
    &\leq 1 - \frac{\mu^{(k+1)}}{\mu^{(k)}} = \frac{\mu^{(k)} - \mu^{(k+1)}}{\mu^{(k)}},
    \end{align*}
    where the inequality follows from \( |z| \geq \mu^{(k+1)} \).
\item \textit{Case 3}: As long as $|z| < \mu^{\mathopen{}\left(k+1\right)\mathclose{}} < \mu^{\mathopen{}\left(k\right)\mathclose{}}$, the difference between the gradients can be computed as 
\begin{align*}
    \left| \frac{z}{\mu^{(k)}} - \frac{z}{\mu^{(k+1)}} \right| &= \left| \frac{\mu^{(k+1)} - \mu^{(k)}}{\mu^{(k)} \mu^{(k+1)}} \right| |z| \leq \frac{\mu^{(k)} - \mu^{(k+1)}}{\mu^{(k)}}, 
    \end{align*}
     where the inequality arises from \( |z| < \mu^{(k+1)} \).
\end{itemize}

In each case, the gradient difference is bounded by $\frac{\mu^{\mathopen{}\left(k\right)\mathclose{}} - \mu^{\mathopen{}\left(k+1\right)\mathclose{}}}{\mu^{\mathopen{}\left(k\right)\mathclose{}}}$. Now, consider the following inequality obtained by the Triangle Inequality: 
$\mathopen{}\left\|\nabla  h\mathopen{}\left(\mathbf{z},{\mu^{\mathopen{}\left(k\right)\mathclose{}}}\right)\mathclose{} -\nabla  h\mathopen{}\left(\mathbf{z},{\mu^{\mathopen{}\left(k+1\right)\mathclose{}}}\right)\mathclose{}\right\|\mathclose{}=\frac{1}{2}\mathopen{}\left\|\sum_{i=1}^n\mathopen{}\left(\nabla  f\mathopen{}\left(z_i,{\mu^{\mathopen{}\left(k\right)\mathclose{}}}\right)\mathclose{} -\nabla  f\mathopen{}\left(z_i,{\mu^{\mathopen{}\left(k+1\right)\mathclose{}}}\right)\mathclose{}\right)\mathclose{}\right\|\mathclose{}\leq \frac{1}{2} \sum_{i=1}^n \mathopen{}\left\|\nabla  f\mathopen{}\left(z_i,{\mu^{\mathopen{}\left(k\right)\mathclose{}}}\right)\mathclose{} -\nabla  f\mathopen{}\left(z_i,{\mu^{\mathopen{}\left(k+1\right)\mathclose{}}}\right)\mathclose{}\right\|\mathclose{}$. Therefore,
given the established upper bound for the individual gradient differences, one can conclude \eqref{eq22}.
\end{proof}

\begin{lemma}\label{lem7} When following algorithm \ref{alg:1}, the change in the dual variable $\boldsymbol{\Psi}$ for each successive iteration can be bounded as follows:
\begin{multline}\label{eq23}
 \mathopen{}\left\|\boldsymbol{\Psi}^{\mathopen{}\left(k+1\right)\mathclose{}}-\boldsymbol{\Psi}^{\mathopen{}\left(k\right)\mathclose{}}\right\|\mathclose{}_2^2  \leq \\ \frac{1}{2\mathopen{}\left(\mu^{\mathopen{}\left(k+1\right)\mathclose{}}\right)\mathclose{}^2} \mathopen{}\left\|\mathbf{z}^{\mathopen{}\left(k+1\right)\mathclose{}}-\mathbf{z}^{\mathopen{}\left(k\right)\mathclose{}}\right\|\mathclose{}_2^2+ \frac{n^2}{2} \mathopen{}\left(\frac{\mu^{\mathopen{}\left(k\right)\mathclose{}}-\mu^{\mathopen{}\left(k+1\right)\mathclose{}}}{\mu^{\mathopen{}\left(k+1\right)\mathclose{}}}\right)\mathclose{}^2.
\end{multline}
\end{lemma}
\begin{proof}
Using the optimality condition in the $\mathbf{z}$ update step of algorithm \ref{alg:1}, we have
$\boldsymbol{\Psi}^{\mathopen{}\left(k+1\right)\mathclose{}}= \nabla h\mathopen{}\left(\mathbf{z}^{\mathopen{}\left(k+1\right)\mathclose{}},{\mu^{\mathopen{}\left(k+1\right)\mathclose{}}}\right)\mathclose{}$.
Subsequently, we bound the squared norm of the difference between $\boldsymbol{\Psi}^{\mathopen{}\left(k+1\right)\mathclose{}}$ and $\boldsymbol{\Psi}^{\mathopen{}\left(k\right)\mathclose{}}$ as follows:
\begin{multline}\label{eq24}
    \mathopen{}\left\|\boldsymbol{\Psi}^{\mathopen{}\left(k+1\right)\mathclose{}}-\boldsymbol{\Psi}^{\mathopen{}\left(k\right)\mathclose{}}\right\|\mathclose{}_2^2\\=\mathopen{}\left\|   \nabla h\mathopen{}\left(\mathbf{z}^{\mathopen{}\left(k+1\right)\mathclose{}},{\mu^{\mathopen{}\left(k+1\right)\mathclose{}}}\right)\mathclose{} -  \nabla h\mathopen{}\left(\mathbf{z}^{\mathopen{}\left(k\right)\mathclose{}},{\mu^{\mathopen{}\left(k\right)\mathclose{}}}\right)\mathclose{}\right\|\mathclose{}_2^2.
\end{multline}
     By applying the triangle inequality on \eqref{eq24} we obtain:
     \begin{multline}\label{eq25}
         \mathopen{}\left \| \nabla h\mathopen{}\left(\mathbf{z}^{\mathopen{}\left(k+1\right)\mathclose{}},{\mu^{\mathopen{}\left(k+1\right)\mathclose{}}}\right)\mathclose{} -  \nabla h\mathopen{}\left(\mathbf{z}^{\mathopen{}\left(k\right)\mathclose{}},{\mu^{\mathopen{}\left(k\right)\mathclose{}}}\right)\mathclose{}\right\|\mathclose{}_2^2 \leq 
         \\
        2 \mathopen{}\left \| \nabla h\mathopen{}\left(\mathbf{z}^{\mathopen{}\left(k+1\right)\mathclose{}},{\mu^{\mathopen{}\left(k+1\right)\mathclose{}}}\right)\mathclose{} -  \nabla h\mathopen{}\left(\mathbf{z}^{\mathopen{}\left(k\right)\mathclose{}},{\mu^{\mathopen{}\left(k+1\right)\mathclose{}}}\right)\mathclose{}\right\|\mathclose{}_2^2 \\ 
      + 2 \mathopen{}\left \| \nabla  h\mathopen{}\left(\mathbf{z}^{\mathopen{}\left(k\right)\mathclose{}},{\mu^{\mathopen{}\left(k+1\right)\mathclose{}}}\right)\mathclose{} -  \nabla h\mathopen{}\left(\mathbf{z}^{\mathopen{}\left(k\right)\mathclose{}},{\mu^{\mathopen{}\left(k\right)\mathclose{}}}\right)\mathclose{}\right\|\mathclose{}_2^2 .
     \end{multline}
      Now by substituting the upper bounds from Lemmas \ref{lem3} and \ref{lem4} in \eqref{eq25}, we can conclude that inequality \eqref{eq23} holds true.
\end{proof}
Lemma \ref{lem5} demonstrates that the derivation in the dual variable $\boldsymbol{\Psi}$ can be bounded by a term composed of the difference in primal variables, and a scale of an equation based on smoothing parameters. This result provides insights into the convergence properties of the optimization algorithm under study. In the following, we provide an upper bound for the total amount of change in one iteration.
\begin{lemma}\label{lem8}
By  $\rho = \frac{1}{\gamma}$ or $\rho = \frac{1}{\gamma-1}$, for MCP or SCAD respectively, since $\sigma_{\Psi}^{\mathopen{}\left(k\right)\mathclose{}}>\frac{2 n \rho}{\min_p \mathopen{}\left \|\mathbf{X}_{:,p}\right\|\mathclose{}_2^2}$
the amount of change in each iteration can be bounded as \begin{multline}\label{eq372}
\bar{\mathcal{L}}_{\sigma_{\Psi}^{\mathopen{}\left(k+1\right)\mathclose{}},\mu^{\mathopen{}\left(k+1\right)\mathclose{}}}\mathopen{}\left(\mathbf{w}^{\mathopen{}\left(k+1\right)\mathclose{}},\mathbf{z}^{\mathopen{}\left(k+1\right)\mathclose{}},\boldsymbol{\Psi}^{\mathopen{}\left(k+1\right)\mathclose{}}\right)\mathclose{}\\-\bar{\mathcal{L}}_{\sigma_{\Psi}^{\mathopen{}\left(k\right)\mathclose{}},\mu^{\mathopen{}\left(k\right)\mathclose{}}}\mathopen{}\left(\mathbf{w}^{\mathopen{}\left(k\right)\mathclose{}},\mathbf{z}^{\mathopen{}\left(k\right)\mathclose{}},\boldsymbol{\Psi}^{\mathopen{}\left(k\right)\mathclose{}}\right)\mathclose{}  \leq -{\xi^{\mathopen{}\left(k+1\right)\mathclose{}}}\mathopen{}\left \|\mathbf{w}^{\mathopen{}\left(k+1\right)\mathclose{}}-\mathbf{w}^{\mathopen{}\left(k\right)\mathclose{}}\right\|\mathclose{}_2^2 \\
    \mathopen{}\left(-\frac{\sigma_{\Psi}^{\mathopen{}\left(k+1\right)\mathclose{}}}{2}+\frac{\sigma_{\Psi}^{\mathopen{}\left(k+1\right)\mathclose{}}}{2\beta^2}\right)\mathclose{}\mathopen{}\left \|\mathbf{z}^{\mathopen{}\left(k+1\right)\mathclose{}}-\mathbf{z}^{\mathopen{}\left(k\right)\mathclose{}}\right\|\mathclose{}_2^2+\\  \frac{\sigma_{\Psi}^{\mathopen{}\left(k+1\right)\mathclose{}}-\sigma_{\Psi}^{\mathopen{}\left(k\right)\mathclose{}}}{2\beta^2} \mathopen{}\left\|\mathbf{z}^{\mathopen{}\left(k\right)\mathclose{}}-\mathbf{z}^{\mathopen{}\left(k-1\right)\mathclose{}}\right\|\mathclose{}_2^2 + \frac{n^2}{2\sigma_{\Psi}^{\mathopen{}\left(k+1\right)\mathclose{}}} \mathopen{}\left(\frac{\mu^{\mathopen{}\left(k\right)\mathclose{}}-\mu^{\mathopen{}\left(k+1\right)\mathclose{}}}{\mu^{\mathopen{}\left(k+1\right)\mathclose{}}}\right)\mathclose{}^2\\+  \frac{n^2\mathopen{}\left(\sigma_{\Psi}^{\mathopen{}\left(k+1\right)\mathclose{}}-\sigma_{\Psi}^{\mathopen{}\left(k\right)\mathclose{}}\right)\mathclose{}}{2 \mathopen{}\left(\sigma_{\Psi}^{\mathopen{}\left(k\right)\mathclose{}}\right)\mathclose{}^2} \mathopen{}\left(\frac{\mu^{k-1}-\mu^{\mathopen{}\left(k\right)\mathclose{}}}{\mu^{\mathopen{}\left(k\right)\mathclose{}}}\right)\mathclose{}^2.
    \end{multline}
\end{lemma}
\begin{proof}
    Consider the equation:
    \begin{multline}\label{eq35}
\bar{\mathcal{L}}_{\sigma_{\Psi}^{\mathopen{}\left(k+1\right)\mathclose{}},\mu^{\mathopen{}\left(k+1\right)\mathclose{}}}\mathopen{}\left(\mathbf{w}^{\mathopen{}\left(k+1\right)\mathclose{}},\mathbf{z}^{\mathopen{}\left(k+1\right)\mathclose{}},\boldsymbol{\Psi}^{\mathopen{}\left(k+1\right)\mathclose{}}\right)\mathclose{}-\\\bar{\mathcal{L}}_{\sigma_{\Psi}^{\mathopen{}\left(k\right)\mathclose{}},\mu^{\mathopen{}\left(k\right)\mathclose{}}}\mathopen{}\left(\mathbf{w}^{\mathopen{}\left(k\right)\mathclose{}},\mathbf{z}^{\mathopen{}\left(k\right)\mathclose{}},\boldsymbol{\Psi}^{\mathopen{}\left(k\right)\mathclose{}}\right)\mathclose{} \\= \eqref{eq22a}+\eqref{eq211a}+\eqref{eq30a}+\eqref{eq261a}+\eqref{eq29a}.\\
    \end{multline}
From Lemma \ref{lem3} and \ref{lem4}, we can get the bounds for each part right-hand side of \eqref{eq35}. Substituting these bounds into the equation, we get:
    \begin{multline}\label{eq371}
\bar{\mathcal{L}}_{\sigma_{\Psi}^{\mathopen{}\left(k+1\right)\mathclose{}},\mu^{\mathopen{}\left(k+1\right)\mathclose{}}}\mathopen{}\left(\mathbf{w}^{\mathopen{}\left(k+1\right)\mathclose{}},\mathbf{z}^{\mathopen{}\left(k+1\right)\mathclose{}},\boldsymbol{\Psi}^{\mathopen{}\left(k+1\right)\mathclose{}}\right)\mathclose{}\\-\bar{\mathcal{L}}_{\sigma_{\Psi}^{\mathopen{}\left(k\right)\mathclose{}},\mu^{\mathopen{}\left(k\right)\mathclose{}}}\mathopen{}\left(\mathbf{w}^{\mathopen{}\left(k\right)\mathclose{}},\mathbf{z}^{\mathopen{}\left(k\right)\mathclose{}},\boldsymbol{\Psi}^{\mathopen{}\left(k\right)\mathclose{}}\right)\mathclose{}\leq   \\
   -\frac{\xi^{\mathopen{}\left(k\right)\mathclose{}}}{2}\mathopen{}\left\|\mathbf{w}^{\mathopen{}\left(k\right)\mathclose{}}-\mathbf{w}^{\mathopen{}\left(k-1\right)\mathclose{}}\right\|\mathclose{}_2^2-\frac{\sigma_{\Psi}^{\mathopen{}\left(k+1\right)\mathclose{}}}{2}\mathopen{}\left\|\mathbf{z}^{\mathopen{}\left(k\right)\mathclose{}}  -\mathbf{z}^{\mathopen{}\left(k-1\right)\mathclose{}}\right\|\mathclose{}_2^2 +\\
   \frac{1}{\sigma_{\Psi}^{\mathopen{}\left(k+1\right)\mathclose{}}}\mathopen{}\left \|\boldsymbol{\Psi}^{\mathopen{}\left(k+1\right)\mathclose{}}-\boldsymbol{\Psi}^{\mathopen{}\left(k\right)\mathclose{}}\right\|\mathclose{}_2^2+\frac{\sigma_{\Psi}^{\mathopen{}\left(k+1\right)\mathclose{}}-\sigma_{\Psi}^{\mathopen{}\left(k\right)\mathclose{}}}{\mathopen{}\left(\sigma_{\Psi}^{\mathopen{}\left(k\right)\mathclose{}}\right)\mathclose{}^2}\mathopen{}\left \|\boldsymbol{\Psi}^{\mathopen{}\left(k\right)\mathclose{}}-\boldsymbol{\Psi}^{k-1}\right\|\mathclose{}_2^2.
    \end{multline}
Further, replacing the bound from inequality \eqref{eq23}, leads the above inequality to \eqref{eq372}. 
\end{proof}

The bound specified in \eqref{eq372} depends on two elements: one tied to the primal variables, and the other to the smoothing parameters. Adjusting parameters can indeed render the primal variable component negative, thereby promoting a decrease in the function's value. Nevertheless, the component stemming from the smoothing parameters remains consistently positive. Despite this, it does not obstruct convergence due to its reducing and summable character. In the following sections, we will establish that the segment relating to the dual smoothing parameter is similarly summable.

\begin{lemma}\label{lem9} By considering $\mu^{\mathopen{}\left(k+1\right)\mathclose{}}=\frac{\beta}{\sigma_{\Psi}^{\mathopen{}\left(k+1\right)\mathclose{}}}$ and $\sigma_{\Psi}^{\mathopen{}\left(k+1\right)\mathclose{}}=c \sqrt{k+1}$, we can coclude:
$ \sum_{k=1}^\infty \frac{n^2}{2}\frac{\mathopen{}\left(\frac{\mu^{\mathopen{}\left(k\right)\mathclose{}}-\mu^{\mathopen{}\left(k+1\right)\mathclose{}}}{\mu^{\mathopen{}\left(k\right)\mathclose{}}}\right)\mathclose{}^2}{\sigma_\Psi^{\mathopen{}\left(k+1\right)\mathclose{}}}<D$ and $\sum_{k=1}^\infty\frac{n^2\mathopen{}\left(\sigma_{\Psi}^{\mathopen{}\left(k+1\right)\mathclose{}}-\sigma_{\Psi}^{\mathopen{}\left(k\right)\mathclose{}}\right)\mathclose{}}{2 \mathopen{}\left(\sigma_{\Psi}^{\mathopen{}\left(k\right)\mathclose{}}\right)\mathclose{}^2} \mathopen{}\left(\frac{\mu^{k-1}-\mu^{\mathopen{}\left(k\right)\mathclose{}}}{\mu^{\mathopen{}\left(k\right)\mathclose{}}}\right)\mathclose{}^2<D'$, where $D$ and $D'$ are postive constants.
\end{lemma}
\begin{proof}
We have $\mu^{\mathopen{}\left(k+1\right)\mathclose{}}=\frac{\beta}{\sigma_{\Psi}^{\mathopen{}\left(k+1\right)\mathclose{}}}$, which implies that $\mu^{\mathopen{}\left(k+1\right)\mathclose{}}\sigma_{\Psi}^{\mathopen{}\left(k+1\right)\mathclose{}}=\beta$. Taking the ratio of consecutive $\mu$ values, we get:

\begin{equation}
\frac{\mu^{\mathopen{}\left(k\right)\mathclose{}}-\mu^{\mathopen{}\left(k+1\right)\mathclose{}}}{\mu^{\mathopen{}\left(k\right)\mathclose{}}} =\frac{\sigma_{\Psi}^{\mathopen{}\left(k+1\right)\mathclose{}}-\sigma_{\Psi}^{\mathopen{}\left(k\right)\mathclose{}}}{\sigma_{\Psi}^{\mathopen{}\left(k+1\right)\mathclose{}}}.
\end{equation}
By setting $\sigma_{\Psi}^{k+1}=c \sqrt{k+1}$ we obtain:
\begin{multline}
\frac{\sigma_{\Psi}^{k+1}-\sigma_{\Psi}^{\mathopen{}\left(k\right)\mathclose{}}}{\sigma_{\Psi}^{k+1}}=  \frac{\sqrt{k+1}-\sqrt{k}}{\sqrt{k+1}}= \\ \frac{\mathopen{}\left(\sqrt{k+1}-\sqrt{k}\right)\mathclose{}\mathopen{}\left(\sqrt{k+1}+\sqrt{k}\right)\mathclose{}}{\sqrt{k+1}\mathopen{}\left(\sqrt{k+1}+\sqrt{k}\right)\mathclose{}}
= \frac{1}{\sqrt{k+1}\mathopen{}\left(\sqrt{k+1}+\sqrt{k}\right)\mathclose{}}\\ < \frac{1}{2 k}.
\end{multline}
Upon squaring both sides, we have:
\begin{equation}\label{eq28}
\mathopen{}\left(\frac{\mu^{\mathopen{}\left(k\right)\mathclose{}}-\mu^{\mathopen{}\left(k+1\right)\mathclose{}}}{\mu^{\mathopen{}\left(k\right)\mathclose{}}}\right)\mathclose{}^2 < \frac{1}{4k^2}.
\end{equation}
Adding up \eqref{eq28} for all $k$ from 1 to $\infty$ and replacing $\sigma_{\Psi}^{\mathopen{}\left(k+1\right)\mathclose{}}$ with $c\sqrt{k+1}$, we derive:
\begin{equation}
\sum_{k=1}^\infty \frac{\mathopen{}\left(\frac{\mu^{\mathopen{}\left(k\right)\mathclose{}}-\mu^{\mathopen{}\left(k+1\right)\mathclose{}}}{\mu^{\mathopen{}\left(k\right)\mathclose{}}}\right)\mathclose{}^2}{\sigma_\Psi^{\mathopen{}\left(k+1\right)\mathclose{}}} \leq  \sum_{k=1}^\infty \frac{\frac{1}{4k^2}}{c\sqrt{k+1}}\leq \sum_{k=1}^\infty \frac{1}{c k^{\frac{5}{2}}}.
\end{equation}
The sum converges to a constant because the summands are decreasing in magnitude, and the sum of the sequence $\frac{1}{k^p}$ is convergent for any $p>1$. Consequently, the sum multiplied by the constant $\frac{n^2}{2}$ can be bounded by a constant $D$:
\begin{equation}
\sum_{k=1}^\infty \frac{n^2}{2} \frac{\mathopen{}\left(\frac{\mu^{\mathopen{}\left(k\right)\mathclose{}}-\mu^{\mathopen{}\left(k+1\right)\mathclose{}}}{\mu^{\mathopen{}\left(k\right)\mathclose{}}}\right)\mathclose{}^2}{\sigma_\Psi^{\mathopen{}\left(k+1\right)\mathclose{}}} < D.
\end{equation}
By following a similar direction we can also obtain:
\begin{equation}
    \sum_{k=1}^\infty\frac{n^2\mathopen{}\left(\sigma_{\Psi}^{\mathopen{}\left(k+1\right)\mathclose{}}-\sigma_{\Psi}^{\mathopen{}\left(k\right)\mathclose{}}\right)\mathclose{}}{2 \mathopen{}\left(\sigma_{\Psi}^{\mathopen{}\left(k\right)\mathclose{}}\right)\mathclose{}^2} \mathopen{}\left(\frac{\mu^{k-1}-\mu^{\mathopen{}\left(k\right)\mathclose{}}}{\mu^{\mathopen{}\left(k\right)\mathclose{}}}\right)\mathclose{}^2<D',
\end{equation}
where $D'$ is a constant.
\end{proof}
The convergence theorem can now be demonstrated based on all privilege lemmas.
 \begin{theorem}[Sufficient Descent Property]\label{lemma12}
Assume there exists a $K$ such that for $k \geq K$, where $\sigma_{\Psi}^{\mathopen{}\left(k+1\right)\mathclose{}} > \frac{2 n \rho}{\min_p |\mathbf{X}{:,p}|^2_2}$ and we update $\sigma_\Psi^{\mathopen{}\left(k+1\right)\mathclose{}} = c\sqrt{k+1}$ with $c > 0$ and $\mu^{\mathopen{}\left(k+1\right)\mathclose{}} = \frac{\beta}{\sigma_\Psi^{\mathopen{}\left(k+1\right)\mathclose{}}}$, where $\beta \geq \sqrt{3}$. Then, if we follow Algorithm \ref{alg:1}, the below relations hold:
\begin{align}
\mathopen{}\left\|\mathbf{w}^{\mathopen{}\left(k+1\right)\mathclose{}} - \mathbf{w}^{\mathopen{}\left(k\right)\mathclose{}}\right\|\mathclose{}_2^2 \in o\mathopen{}\left({k^{-\frac{3}{2}}}\right)\mathclose{} \\
\mathopen{}\left\|\mathbf{z}^{\mathopen{}\left(k+1\right)\mathclose{}} - \mathbf{z}^{\mathopen{}\left(k\right)\mathclose{}}\right\|\mathclose{}_2^2 \in o\mathopen{}\left({k^{-\frac{3}{2}}}\right)\mathclose{} \\
\lim_{k \rightarrow \infty} \sigma_\Psi^{\mathopen{}\left(k+1\right)\mathclose{}} \mathopen{}\left\|\mathbf{z}^{\mathopen{}\left(k+1\right)\mathclose{}} + \mathbf{X} \mathbf{w}^{\mathopen{}\left(k+1\right)\mathclose{}} - \mathbf{y}\right\|\mathclose{}_2^2 \in o\mathopen{}\left({k^{-1}}\right)\mathclose{}.
\end{align}
 \end{theorem}
\begin{proof}
Consider a $K'$ greater than $K$ and, without loss of generality, let $K$ be zero. The overall change in the augmented Lagrangian from iteration $K$ to iteration $K'$, using Lemma \ref{lem8} for single iteration changes and Lemma \ref{lem9} for the bounds, is given as:
\begin{multline}
\bar{\mathcal{L}}_{\sigma_{\Psi}^{K'},\mu^{K'}}\mathopen{}\left(\mathbf{w}^{K'},\mathbf{z}^{K'},\boldsymbol{\Psi}^{K'}\right)\mathclose{}-\bar{\mathcal{L}}_{\sigma_{\Psi}^{0},\mu^{0}}\mathopen{}\left(\mathbf{w}^{0},\mathbf{z}^{0},\boldsymbol{\Psi}^{0}\right)\mathclose{} \leq \\  \sum_{k=0}^{K'-1}  \Bigg(
    \frac{\xi^{\mathopen{}\left(k+1\right)\mathclose{}}}{2}\mathopen{}\left\|\mathbf{w}^{\mathopen{}\left(k+1\right)\mathclose{}}-\mathbf{w}^{\mathopen{}\left(k\right)\mathclose{}}\right\|\mathclose{}_2^2-\frac{\sigma_{\Psi}^{\mathopen{}\left(k+1\right)\mathclose{}}}{2}\mathopen{}\left\|\mathbf{z}^{\mathopen{}\left(k+1\right)\mathclose{}}-\mathbf{z}^{\mathopen{}\left(k\right)\mathclose{}}\right\|\mathclose{}_2^2 +\\ \frac{\mathopen{}\left(\sigma_{\Psi}^{\mathopen{}\left(k+1\right)\mathclose{}}\right)\mathclose{}}{2\beta^2} \mathopen{}\left\|\mathbf{z}^{\mathopen{}\left(k+1\right)\mathclose{}}-\mathbf{z}^{\mathopen{}\left(k\right)\mathclose{}}\right\|\mathclose{}_2^2+ \frac{\mathopen{}\left(\sigma_{\Psi}^{\mathopen{}\left(k+1\right)\mathclose{}}-\sigma_{\Psi}^{\mathopen{}\left(k\right)\mathclose{}}\right)\mathclose{}}{2\beta^2} \mathopen{}\left\|\mathbf{z}^{\mathopen{}\left(k\right)\mathclose{}}-\mathbf{z}^{\mathopen{}\left(k-1\right)\mathclose{}}\right\|\mathclose{}_2^2   \\  +\frac{n^2}{2\mathopen{}\left(\sigma_{\Psi}^{\mathopen{}\left(k+1\right)\mathclose{}}\right)\mathclose{}} \mathopen{}\left(\frac{\mu^{\mathopen{}\left(k\right)\mathclose{}}-\mu^{\mathopen{}\left(k+1\right)\mathclose{}}}{\mu^{\mathopen{}\left(k+1\right)\mathclose{}}}\right)\mathclose{}^2 +   \\ \frac{\mathopen{}\left(\sigma_{\Psi}^{\mathopen{}\left(k+1\right)\mathclose{}}-\sigma_{\Psi}^{\mathopen{}\left(k\right)\mathclose{}}\right)\mathclose{} n^2}{2 \mathopen{}\left(\sigma_{\Psi}^{\mathopen{}\left(k\right)\mathclose{}}\right)\mathclose{}^2} \mathopen{}\left(\frac{\mu^{k-1}-\mu^{\mathopen{}\left(k\right)\mathclose{}}}{\mu^{\mathopen{}\left(k\right)\mathclose{}}}\right)\mathclose{}^2 \Bigg)\leq   D +D' - S_{K'},
\end{multline}
where, $D$ and $D'$ are constants based on Lemma \ref{lem9}, and the value $S_{K'}$ is defined as:
\begin{multline}
S_{K'}=\sum_{k=1}^{K'-1} \eta^{\mathopen{}\left(k+1\right)\mathclose{}}\mathopen{}\left\|\mathbf{z}^{\mathopen{}\left(k+1\right)\mathclose{}} -   \mathbf{z}^{\mathopen{}\left(k\right)\mathclose{}}\right\|\mathclose{}_2^2 \\+ {\xi^{\mathopen{}\left(k+1\right)\mathclose{}}}\mathopen{}\left\|\mathbf{w}^{\mathopen{}\left(k+1\right)\mathclose{}}-\mathbf{w}^{\mathopen{}\left(k\right)\mathclose{}}\right\|\mathclose{}_2^2,
      \end{multline}
 where $\eta^{\mathopen{}\left(k+1\right)\mathclose{}}=\mathopen{}\left(\frac{\sigma_{\Psi}^{\mathopen{}\left(k+1\right)\mathclose{}}}{2}-\frac{\sigma_{\Psi}^{\mathopen{}\left(k+1\right)\mathclose{}}}{\beta^2}-\frac{\sigma_{\Psi}^{\mathopen{}\left(k+1\right)\mathclose{}}}{2\beta^2}\right)\mathclose{}$. 
 
 Given that the augmented Lagrangian is lower-bounded by Lemma \ref{lem2}, and assuming $\sigma_{\Psi}^{0}\geq c$ and $\beta \geq \sqrt{3}$, each element in the summation $S_{K'}$ is non-negative. Thus, we have $\lim_{K\rightarrow \infty} 0 \leq S_{K'}< \infty.$ As a result, we obtain $\lim_{k\rightarrow \infty} \frac{\xi_{\lambda}^{\mathopen{}\left(k+1\right)\mathclose{}}}{2}\mathopen{}\left\|\mathbf{w}^{\mathopen{}\left(k+1\right)\mathclose{}}-\mathbf{w}^{\mathopen{}\left(k\right)\mathclose{}}\right\|\mathclose{}_2^2=0 $ and $\lim_{k\rightarrow \infty} \eta^{\mathopen{}\left(k\right)\mathclose{}} \mathopen{}\left\|\mathbf{z}^{\mathopen{}\left(k+1\right)\mathclose{}}-\mathbf{z}^{\mathopen{}\left(k\right)\mathclose{}}\right\|\mathclose{}_2^2 =0$. Considering that $\eta^{\mathopen{}\left(k+1\right)\mathclose{}}$ and ${\xi^{\mathopen{}\left(k+1\right)\mathclose{}}}$ are positive terms that increase at rates of $\Omega\mathopen{}\left(\sqrt{k}\right)\mathclose{}$, $\mathopen{}\left\|\mathbf{w}^{\mathopen{}\left(k+1\right)\mathclose{}}-\mathbf{w}^{\mathopen{}\left(k\right)\mathclose{}}\right\|\mathclose{}_2^2$ and $\mathopen{}\left\|\mathbf{z}^{\mathopen{}\left(k+1\right)\mathclose{}}-\mathbf{z}^{\mathopen{}\left(k\right)\mathclose{}}\right\|\mathclose{}_2^2$ decrease at rates of $o\mathopen{}\left(k^{-\frac{3}{2}}\right)\mathclose{}$.

 Finally, from Lemma \ref{lem3}, we get $\sigma_{\Psi}^{\mathopen{}\left(k+1\right)\mathclose{}}\mathopen{}\left\|\mathbf{z}^{\mathopen{}\left(k+1\right)\mathclose{}}+ \mathbf{X}\mathbf{w}^{\mathopen{}\left(k+1\right)\mathclose{}}- \mathbf{y}\right\|\mathclose{}_2^2=\frac{\mathopen{}\left\|\boldsymbol{\Psi}^{\mathopen{}\left(k+1\right)\mathclose{}}-\boldsymbol{\Psi}^{\mathopen{}\left(k\right)\mathclose{}}\right\|\mathclose{}_2^2}{\sigma_{\Psi}^{\mathopen{}\left(k+1\right)\mathclose{}}}$. Using Lemma \ref{lem7}, we bound it above as $ \frac{\sigma_{\Psi}^{\mathopen{}\left(k+1\right)\mathclose{}}}{2\beta^2} \mathopen{}\left\|\mathbf{z}^{\mathopen{}\left(k+1\right)\mathclose{}}-\mathbf{z}^{\mathopen{}\left(k\right)\mathclose{}}\right\|\mathclose{}_2^2+ \frac{n^2}{2\sigma_{\Psi}^{\mathopen{}\left(k+1\right)\mathclose{}}} \mathopen{}\left(\frac{\mu^{\mathopen{}\left(k\right)\mathclose{}}-\mu^{\mathopen{}\left(k+1\right)\mathclose{}}}{\mu^{\mathopen{}\left(k+1\right)\mathclose{}}}\right)\mathclose{}^2 $. Given that $\frac{\sigma_{\Psi}^{\mathopen{}\left(k+1\right)\mathclose{}}}{2\beta^2} \mathopen{}\left\|\mathbf{z}^{\mathopen{}\left(k+1\right)\mathclose{}}-\mathbf{z}^{\mathopen{}\left(k\right)\mathclose{}}\right\|\mathclose{}_2^2 \in o\mathopen{}\left({k^{-1}}\right)\mathclose{}$ and $\frac{n^2}{2\sigma_{\Psi}^{\mathopen{}\left(k+1\right)\mathclose{}}} \mathopen{}\left(\frac{\mu^{\mathopen{}\left(k\right)\mathclose{}}-\mu^{\mathopen{}\left(k+1\right)\mathclose{}}}{\mu^{\mathopen{}\left(k+1\right)\mathclose{}}}\right)\mathclose{}^2\in o\mathopen{}\left({k^{-\frac{5}{2}}}\right)\mathclose{}$, it follows that $\sigma_{\Psi}^{\mathopen{}\left(k+1\right)\mathclose{}}\mathopen{}\left\|\mathbf{z}^{\mathopen{}\left(k+1\right)\mathclose{}}+ \mathbf{X}\mathbf{w}^{\mathopen{}\left(k+1\right)\mathclose{}}- \mathbf{y}\right\|\mathclose{}_2^2 \in o\mathopen{}\left({k^{-1}}\right)\mathclose{}$. This concludes our proof.
\end{proof}
\begin{remark}
It is worth noting that, Theorem \ref{lemma12} gives us this freedom to update $\sigma_\Psi^{\mathopen{}\left(k+1\right)\mathclose{}}$ and $\beta^{\mathopen{}\left(k+1\right)\mathclose{}}$ in a different setting before $K$, which might help us in the practice.
\end{remark}
Theorem \ref{lemma12} confirms the convergence of the algorithm to a specified limit point.  Following this, we explore further considerations, such as the rate of convergence, the optimality of the outcome, and the broader implications for global convergence, which are considered in detail in the upcoming theorems.

\begin{theorem}[Sub-gradient Bound Property]\label{lemma13}
Suppose a constant $K$ exists. For $k \geq K$, we update $\sigma_\Psi^{\mathopen{}\left(k+1\right)\mathclose{}} = c\sqrt{k+1}$ where $c > 0$, and $\mu^{\mathopen{}\left(k+1\right)\mathclose{}} = \frac{\beta}{\sigma_\Psi^{\mathopen{}\left(k+1\right)\mathclose{}}}$ with $\beta \geq \sqrt{3}$. Let $\mathbf{d}^{\mathopen{}\left(k+1\right)\mathclose{}}$ denote the gradient of the augmented Lagrangian $\bar{\mathcal{L}}{\sigma_{\Psi}^{\mathopen{}\left(k+1\right)\mathclose{}},\mu^{\mathopen{}\left(k+1\right)\mathclose{}}}$ at point $(\mathbf{w}^{\mathopen{}\left(k+1\right)\mathclose{}},\mathbf{z}^{\mathopen{}\left(k+1\right)\mathclose{}},\boldsymbol{\Psi}^{\mathopen{}\left(k+1\right)\mathclose{}})$. Given these definitions, if Algorithm 1 is followed, it can be proven that $\lim_{k\rightarrow \infty}\mathopen{}\left\|\mathbf{d}^{\mathopen{}\left(k+1\right)\mathclose{}}\right\|\mathclose{}\in o\mathopen{}\left({k^{-\frac{1}{4}}}\right)\mathclose{}$, which means that the gradient approaches zero as $k$ increases, at a rate of $o\mathopen{}\left({k^{-\frac{1}{4}}}\right)\mathclose{}$.
\end{theorem}
\begin{proof}
We start with the sub-gradients of the augmented Lagrangian with respect to each optimization variable.

For $w_p$, we have the following relation:
    \begin{multline}\label{eq:derwp}
d_{w_p}^{\mathopen{}\left(k+1\right)\mathclose{}} \in  \nabla_{w_p} \bar{\mathcal{L}}_{\sigma_{\Psi}^{\mathopen{}\left(k+1\right)\mathclose{}},\mu^{\mathopen{}\left(k+1\right)\mathclose{}}}\mathopen{}\left(\mathbf{w}^{\mathopen{}\left(k+1\right)\mathclose{}},\mathbf{z}^{\mathopen{}\left(k+1\right)\mathclose{}},\boldsymbol{\Psi}^{\mathopen{}\left(k+1\right)\mathclose{}}\right)\mathclose{}\\=  \partial_{w_p} g\mathopen{}\left(w_p^{\mathopen{}\left(k+1\right)\mathclose{}}\right)\mathclose{} +\mathbf{X}_{:,p}^{\text{T}}{\boldsymbol{\Psi}^{\mathopen{}\left(k\right)\mathclose{}}}  +\mathbf{X}_{:,p}^{\text{T}}\mathopen{}\left(\boldsymbol{\Psi}^{\mathopen{}\left(k+1\right)\mathclose{}} -\boldsymbol{\Psi}^{\mathopen{}\left(k\right)\mathclose{}}\right)\mathclose{}\\ +\sigma_{\Psi}^{\mathopen{}\left(k+1\right)\mathclose{}}\mathbf{X}_{:,p}^{\text{T}}\mathopen{}\left(\mathbf{X}_{:,\leq p}\mathbf{w}_{\leq p}^{\mathopen{}\left(k+1\right)\mathclose{}}+\mathbf{X}_{:,> p}\mathbf{w}_{> p}^{\mathopen{}\left(k\right)\mathclose{}}+\mathbf{z}^{\mathopen{}\left(k\right)\mathclose{}}-\mathbf{y}\right)\mathclose{}\\+\sigma_{\Psi}^{\mathopen{}\left(k+1\right)\mathclose{}}\mathbf{X}_{:,p}^{\text{T}}\mathopen{}\left(\mathbf{X}_{:,> p}\mathbf{w}_{> p}^{\mathopen{}\left(k+1\right)\mathclose{}}-\mathbf{X}_{:,> p}\mathbf{w}_{> p}^{\mathopen{}\left(k\right)\mathclose{}}+\mathbf{z}^{\mathopen{}\left(k+1\right)\mathclose{}}-\mathbf{z}^{\mathopen{}\left(k\right)\mathclose{}}\right)\mathclose{},
\end{multline}
where $g\mathopen{}\left(w_p^{\mathopen{}\left(k+1\right)\mathclose{}}\right)\mathclose{}$ denotes the gradient of the loss function with respect to $w_p$, evaluated at $w_p^{\mathopen{}\left(k+1\right)\mathclose{}}$.

    From the optimally-condition in $w_p$ obtain step, we obtain:
     \begin{multline}\label{eq:optwp}
     0 \in \partial_{w_p} g\mathopen{}\left(w_p^{\mathopen{}\left(k+1\right)\mathclose{}}\right)\mathclose{}+\mathbf{X}_{:,p}^{\text{T}}{\boldsymbol{\Psi}^{\mathopen{}\left(k\right)\mathclose{}}}+\\ \sigma_{\Psi}^{\mathopen{}\left(k+1\right)\mathclose{}}\mathbf{X}_{:,p}^{\text{T}} \mathopen{}\left(\mathbf{X}_{:,\leq p}\mathbf{w}_{\leq p}^{\mathopen{}\left(k+1\right)\mathclose{}}+  \mathbf{X}_{:,> p}\mathbf{w}_{> p}^{\mathopen{}\left(k\right)\mathclose{}}+\mathbf{z}^{\mathopen{}\left(k\right)\mathclose{}}-\mathbf{y}\right)\mathclose{}. \end{multline}   
    Therefore, combining \eqref{eq:derwp} and \eqref{eq:optwp} leads to the definition:
     \begin{multline*}
     d_{w_p}^{\mathopen{}\left(k+1\right)\mathclose{}} := \mathbf{X}_{:,p}^{\text{T}}\mathopen{}\left(\boldsymbol{\Psi}^{\mathopen{}\left(k+1\right)\mathclose{}}-\boldsymbol{\Psi}^{\mathopen{}\left(k\right)\mathclose{}}\right)\mathclose{} +\\\sigma_{\Psi}^{\mathopen{}\left(k+1\right)\mathclose{}}\mathbf{X}_{:,p}^{\text{T}}\Big(\mathbf{X}_{:,> p}\mathbf{w}_{> p}^{\mathopen{}\left(k+1\right)\mathclose{}} -\mathbf{X}_{:,> p}\mathbf{w}_{> p}^{\mathopen{}\left(k\right)\mathclose{}}+\mathbf{z}^{\mathopen{}\left(k+1\right)\mathclose{}}-\mathbf{z}^{\mathopen{}\left(k\right)\mathclose{}}\Big).
     \end{multline*}
   Now we can estimate an upper bound for the norm of $d_{w_p}^{\mathopen{}\left(k+1\right)\mathclose{}}$ as follows:
     \begin{multline}\label{eq:bwp}
\mathopen{}\left\|d_{w_p}^{\mathopen{}\left(k+1\right)\mathclose{}}\right\|\mathclose{} \leq  \max_p \mathopen{}\left\|\mathbf{X}_{:,p}\right\|\mathclose{} \mathopen{}\left\|\boldsymbol{\Psi}^{\mathopen{}\left(k+1\right)\mathclose{}}-\boldsymbol{\Psi}^{\mathopen{}\left(k\right)\mathclose{}}\right\|\mathclose{}+ \\ \sigma_{\Psi}^{\mathopen{}\left(k+1\right)\mathclose{}}P\max_p\mathopen{}\left\|\mathbf{X}_{:,p}\right\|\mathclose{}_2^2\mathopen{}\left\|\mathbf{w}^{\mathopen{}\left(k+1\right)\mathclose{}}-\mathbf{w}^{\mathopen{}\left(k\right)\mathclose{}}\right\|\mathclose{}+\\ \max_p\mathopen{}\left\|\mathbf{X}_{:,p}\right\|\mathclose{}\mathopen{}\left\|\mathbf{z}^{\mathopen{}\left(k+1\right)\mathclose{}}-\mathbf{z}^{\mathopen{}\left(k\right)\mathclose{}}\right\|\mathclose{}.\end{multline}  
    By scaling the right-hand side of \eqref{eq:bwp} with $P$ and setting $v=P\max_p\mathopen{}\left\|\mathbf{X}_{:,p}\right\|\mathclose{}$, we can derive an upper bound for the norm of the derivative with respect to $\mathbf{w}$ at the iteration $k+1$:
    \begin{multline}\label{upwp}
    \mathopen{}\left\|\mathbf{d}_{\mathbf{w}}^{\mathopen{}\left(k+1\right)\mathclose{}}\right\|\mathclose{}\leq v \mathopen{}\left\|\boldsymbol{\Psi}^{\mathopen{}\left(k+1\right)\mathclose{}}-\boldsymbol{\Psi}^{\mathopen{}\left(k\right)\mathclose{}}\right\|\mathclose{}\\ +  \sigma_{\Psi}^{\mathopen{}\left(k+1\right)\mathclose{}}v^2\mathopen{}\left\|\mathbf{w}^{\mathopen{}\left(k+1\right)\mathclose{}}-\mathbf{w}^{\mathopen{}\left(k\right)\mathclose{}}\right\|\mathclose{}\\ + v\mathopen{}\left\|\mathbf{z}^{\mathopen{}\left(k+1\right)\mathclose{}}-\mathbf{z}^{\mathopen{}\left(k\right)\mathclose{}}\right\|\mathclose{}. \end{multline}
    By substituting results from Lemma \ref{lem7} in \eqref{upwp}, we further simplify:
    \begin{multline*}
    \mathopen{}\left\|\mathbf{d}_{\mathbf{w}}^{\mathopen{}\left(k+1\right)\mathclose{}}\right\|\mathclose{} \leq v\mathopen{}\left(\frac{\sigma_{\Psi}^{\mathopen{}\left(k+1\right)\mathclose{}}}{2\beta}+1\right)\mathclose{}\mathopen{}\left\|\mathbf{z}^{\mathopen{}\left(k+1\right)\mathclose{}}-\mathbf{z}^{\mathopen{}\left(k\right)\mathclose{}}\right\|\mathclose{}+\\ \sigma_{\Psi}^{\mathopen{}\left(k+1\right)\mathclose{}}v^2\mathopen{}\left\|\mathbf{w}^{\mathopen{}\left(k+1\right)\mathclose{}}-\mathbf{w}^{\mathopen{}\left(k\right)\mathclose{}}\right\|\mathclose{} +\frac{v \mathopen{}\left(\sigma_{\Psi}^{\mathopen{}\left(k+1\right)\mathclose{}}-\sigma_{\Psi}^{\mathopen{}\left(k\right)\mathclose{}} \right)\mathclose{}n}{2\sigma_{\Psi}^{\mathopen{}\left(k+1\right)\mathclose{}}}.
    \end{multline*}

Similarly, the derivatives with respect to $\mathbf{z}$ are:
    \begin{multline*}\mathbf{d}_{\mathbf{z}}^{\mathopen{}\left(k+1\right)\mathclose{}}=\nabla_{\mathbf{z}} \bar{\mathcal{L}}_{\sigma_{\Psi}^{\mathopen{}\left(k+1\right)\mathclose{}},\mu^{\mathopen{}\left(k+1\right)\mathclose{}}}\mathopen{}\left(\mathbf{w}^{\mathopen{}\left(k+1\right)\mathclose{}},\mathbf{z}^{\mathopen{}\left(k+1\right)\mathclose{}},\boldsymbol{\Psi}^{\mathopen{}\left(k+1\right)\mathclose{}}\right)\mathclose{}\\= \boldsymbol{\Psi}^{\mathopen{}\left(k+1\right)\mathclose{}}-\boldsymbol{\Psi}^{\mathopen{}\left(k\right)\mathclose{}},\end{multline*}
    which  substituting results from Lemma \ref{lem7} leads to the following bound:
    \begin{multline}\label{eq:bz}
   \mathopen{}\left\|\mathbf{d}_{\mathbf{z}}^{\mathopen{}\left(k+1\right)\mathclose{}} \right\|\mathclose{} = \mathopen{}\left\|\boldsymbol{\Psi}^{\mathopen{}\left(k+1\right)\mathclose{}}-\boldsymbol{\Psi}^{\mathopen{}\left(k+1\right)\mathclose{}}\right\|\mathclose{}
    \leq \\ \frac{\sigma_{\Psi}^{\mathopen{}\left(k+1\right)\mathclose{}}}{2\beta}\mathopen{}\left\|\mathbf{z}^{\mathopen{}\left(k+1\right)\mathclose{}}-\mathbf{z}^{\mathopen{}\left(k\right)\mathclose{}}\right\|\mathclose{}+\frac{n}{2}\mathopen{}\left(\frac{\mu^{\mathopen{}\left(k\right)\mathclose{}}-\mu^{\mathopen{}\left(k+1\right)\mathclose{}}}{\mu^{\mathopen{}\left(k\right)\mathclose{}}}\right)\mathclose{} \\ \leq \frac{\sigma_{\Psi}^{\mathopen{}\left(k+1\right)\mathclose{}}}{2\beta}\mathopen{}\left\|\mathbf{z}^{\mathopen{}\left(k+1\right)\mathclose{}}-\mathbf{z}^{\mathopen{}\left(k\right)\mathclose{}}\right\|\mathclose{}+\frac{\mathopen{}\left(\sigma_{\Psi}^{\mathopen{}\left(k+1\right)\mathclose{}}-\sigma_{\Psi}^{\mathopen{}\left(k\right)\mathclose{}}\right)\mathclose{} n}{2\sigma_{\Psi}^{\mathopen{}\left(k+1\right)\mathclose{}}}.\end{multline}

    For the derivative of the approximate augmented Lagrangian with respect to $\boldsymbol{\Psi}$ in iteration $k+1$, we have:
\begin{multline*}\mathbf{d}_{\boldsymbol{\Psi}}^{\mathopen{}\left(k+1\right)\mathclose{}}=\nabla_{\boldsymbol{\Psi}} \bar{\mathcal{L}}_{\sigma_{\Psi}^{\mathopen{}\left(k+1\right)\mathclose{}},\mu^{\mathopen{}\left(k+1\right)\mathclose{}}}\mathopen{}\left(\mathbf{w}^{\mathopen{}\left(k+1\right)\mathclose{}},\mathbf{z}^{\mathopen{}\left(k+1\right)\mathclose{}},\boldsymbol{\Psi}^{\mathopen{}\left(k+1\right)\mathclose{}}\right)\mathclose{}= \\ \mathbf{z}^{\mathopen{}\left(k+1\right)\mathclose{}}+\mathbf{X}\mathbf{z}^{\mathopen{}\left(k+1\right)\mathclose{}}-\mathbf{y}=\frac{\boldsymbol{\Psi}^{\mathopen{}\left(k+1\right)\mathclose{}}-\boldsymbol{\Psi}^{\mathopen{}\left(k\right)\mathclose{}}}{\sigma_{\Psi}^{\mathopen{}\left(k+1\right)\mathclose{}}}.\end{multline*}
    Subsequently, the norm of the derivative can be bounded by
    \begin{multline}\label{eq:psi}
\mathopen{}\left\|\mathbf{d}_{\boldsymbol{\Psi}}^{\mathopen{}\left(k+1\right)\mathclose{}}\right\|\mathclose{}=\frac{\mathopen{}\left\|\boldsymbol{\Psi}^{\mathopen{}\left(k+1\right)\mathclose{}}-\boldsymbol{\Psi}^{\mathopen{}\left(k\right)\mathclose{}}\right\|\mathclose{}}{\sigma_{\Psi}^{\mathopen{}\left(k+1\right)\mathclose{}}}\leq \frac{1}{2\beta
    }\mathopen{}\left\|\mathbf{z}^{\mathopen{}\left(k+1\right)\mathclose{}}-\mathbf{z}^{\mathopen{}\left(k\right)\mathclose{}}\right\|\mathclose{}\\+\frac{\mathopen{}\left(\sigma_{\Psi}^{\mathopen{}\left(k+1\right)\mathclose{}}-\sigma_{\Psi}^{\mathopen{}\left(k\right)\mathclose{}}\right)\mathclose{} n}{2\mathopen{}\left(\sigma_{\Psi}^{\mathopen{}\left(k+1\right)\mathclose{}}\right)\mathclose{}^2}.\end{multline}
Finally, by substituting the results from \eqref{upwp}, \eqref{eq:bz}, and \eqref{eq:psi} into the definition of $\|\mathbf{d}^{(k+1)}\|=\mathopen{}\left\|\mathopen{}\left[\mathbf{d}_{\mathbf{w}}^{(k+1)},\mathbf{d}_{\mathbf{z}}^{(k+1)},\mathbf{d}_{\boldsymbol{\Psi}}^{(k+1)}\right]\right\|\mathclose{}$, we get: 
    \begin{multline}\label{eq:optt}    \mathopen{}\left\|\mathbf{d}^{\mathopen{}\left(k+1\right)\mathclose{}}\right\|\mathclose{}\leq \sigma_{\Psi}^{\mathopen{}\left(k+1\right)\mathclose{}}v^2\mathopen{}\left\|\mathbf{w}^{\mathopen{}\left(k+1\right)\mathclose{}}-\mathbf{w}^{\mathopen{}\left(k\right)\mathclose{}}\right\|\mathclose{}
    \\ +\mathopen{}\left(\frac{1}{2\beta}+\frac{\sigma_{\Psi}^{\mathopen{}\left(k+1\right)\mathclose{}}}{2\beta}+v\mathopen{}\left(\frac{\sigma_{\Psi}^{\mathopen{}\left(k+1\right)\mathclose{}}}{2\beta}+1\right)\mathclose{}\right)\mathclose{}\mathopen{}\left\|\mathbf{z}^{\mathopen{}\left(k+1\right)\mathclose{}}-\mathbf{z}^{\mathopen{}\left(k\right)\mathclose{}}\right\|\mathclose{}\\
    +\mathopen{}\left(v+1+\frac{1}{\sigma_{\Psi}^{\mathopen{}\left(k+1\right)\mathclose{}}}\right)\mathclose{}\frac{\mathopen{}\left(\sigma_{\Psi}^{\mathopen{}\left(k+1\right)\mathclose{}}-\sigma_{\Psi}^{\mathopen{}\left(k\right)\mathclose{}}\right)\mathclose{}n}{2\sigma_{\Psi}^{\mathopen{}\left(k+1\right)\mathclose{}}}.\end{multline}

    Consider the following terms:
    \begin{itemize}
        \item The norms $\mathopen{}\left\|\mathbf{z}^{\mathopen{}\left(k+1\right)\mathclose{}}-\mathbf{z}^{\mathopen{}\left(k\right)\mathclose{}}\right\|\mathclose{}$ and $\mathopen{}\left\|\mathbf{w}^{\mathopen{}\left(k+1\right)\mathclose{}}-\mathbf{w}^{\mathopen{}\left(k\right)\mathclose{}}\right\|\mathclose{}$ belong to the order $o\mathopen{}\left({k^{-\frac{3}{4}}}\right)\mathclose{}$, resulting from Theorem \ref{lemma12}.
        \item The expression $\mathopen{}\left(\frac{1}{2\beta}+\frac{\sigma_{\Psi}^{\mathopen{}\left(k+1\right)\mathclose{}}}{2\beta}+v\mathopen{}\left(\frac{\sigma_{\Psi}^{\mathopen{}\left(k+1\right)\mathclose{}}}{2\beta}+1\right)\mathclose{}\right)\mathclose{}$, and $\sigma_{\Psi}^{\mathopen{}\left(k+1\right)\mathclose{}} v^2$  belong to the order $\Omega\mathopen{}\left(k^{\frac{1}{2}}\right)\mathclose{}$.
        \item The term $\mathopen{}\left(v+1+\frac{1}{\sigma_{\Psi}^{\mathopen{}\left(k+1\right)\mathclose{}}}\right)\mathclose{}\frac{\mathopen{}\left(\sigma_{\Psi}^{\mathopen{}\left(k+1\right)\mathclose{}}-\sigma_{\Psi}^{\mathopen{}\left(k\right)\mathclose{}}\right)\mathclose{}n}{2\sigma_{\Psi}^{\mathopen{}\left(k+1\right)\mathclose{}}}$ is of the order $O\mathopen{}\left(k^{-1}\right)\mathclose{}$.
    \end{itemize}
    With these orders, from the right-hand side of \eqref{eq:optt} we can infer that the norm $\mathopen{}\left\|\mathbf{d}^{\mathopen{}\left(k+1\right)\mathclose{}}\right\|\mathclose{}$ converges to zero at the order $o\mathopen{}\left({k^{-\frac{1}{4}}}\right)\mathclose{}$. Thus, the proof is completed.
\end{proof}
Having elucidated Theorem \ref{lemma13}, which effectively provides the optimality condition and the convergence rate, we will now embark on the next critical phase of our discussion. This entails presenting the global convergence theorem, which affirms that all limit points yielded by the algorithm sequence qualify as stationary points that satisfy the KKT conditions.
\begin{theorem}[Global Convergence]\label{theorem1}
Suppose a constant $K$ exists such that, for every $k \geq K$, $\sigma_\Psi^{\mathopen{}\left(k+1\right)\mathclose{}}$ is updated as $c\sqrt{k+1}$ with $c > 0$, and $\mu^{\mathopen{}\left(k+1\right)\mathclose{}}$ is set to $\frac{\beta}{\sigma_\Psi^{\mathopen{}\left(k+1\right)\mathclose{}}}$, where $\beta \geq \sqrt{3}$. Then, Algorithm \ref{alg:1} will converge to a stationary point $\mathopen{}\left(\mathbf{w}^{*},\mathbf{z}^{*},\boldsymbol{\Psi}^{*}\right)\mathclose{}$ that satisfies the KKT conditions:
\begin{subequations}
\begin{align}
&\mathbf{X}^{\textbf{T}}\boldsymbol{\Psi}^{*}\in n \partial P_{\lambda,\gamma}\mathopen{}\left(\mathbf{w}\right)\mathclose{}, \\
& \boldsymbol{\Psi}^{*} \in \partial \rho_{\tau}\mathopen{}\left(\mathbf{z}^{*}\right)\mathclose{}, \\
& \mathbf{z}^{*}+\mathbf{X}\mathbf{w}^{*}-\mathbf{y}=0.
\end{align}
\end{subequations}
\end{theorem}
\begin{proof}
The continuity of the approximate augmented Lagrangian $\bar{\mathcal{L}}_{\sigma{\Psi}^{\mathopen{}\left(k\right)\mathclose{}},\mu^{\mathopen{}\left(k\right)\mathclose{}}}\mathopen{}\left(\mathbf{w}^{\mathopen{}\left(k\right)\mathclose{}},\mathbf{z}^{\mathopen{}\left(k\right)\mathclose{}},\boldsymbol{\Psi}^{\mathopen{}\left(k\right)\mathclose{}}\right)\mathclose{}$ with respect to each of its inputs allows us to examine a converging subsequence $\mathopen{}\left(\mathbf{w}^{(k_j)}, \mathbf{z}^{(k_j)}, \boldsymbol{\Psi}^{(k_j)}\right)\mathclose{}, j\geq0$, produced by Algorithm \ref{alg:1}. Therefore, assuming $\mathopen{}\left(\mathbf{w}^*, \mathbf{z}^*, \boldsymbol{\Psi}^*\right)\mathclose{}$ is the limit point of this subsequence, we can observe:
\begin{multline}
    \lim_{j\rightarrow\infty} \bar{\mathcal{L}}_{\sigma{\Psi}^{(k_j)},\mu^{(k_j)}}\mathopen{}\left(\mathbf{w}^{(k_j)}, \mathbf{z}^{(k_j)}, \boldsymbol{\Psi}^{(k_j)}\right)\mathclose{}=\\ \lim_{k\rightarrow\infty}\bar{\mathcal{L}}_{\sigma{\Psi}^{\mathopen{}\left(k\right)\mathclose{}},\mu^{\mathopen{}\left(k\right)\mathclose{}}}\mathopen{}\left(\mathbf{w}^*,\mathbf{z}^*,\boldsymbol{\Psi}^*\right)\mathclose{}.
\end{multline}

Next, using the result of Theorem \ref{lemma12}, we have $\lim_{k\rightarrow\infty} \sigma_{\Psi}^{\mathopen{}\left(k\right)\mathclose{}}\mathopen{}\left\|\mathbf{z}^{\mathopen{}\left(k\right)\mathclose{}}+ \mathbf{X}\mathbf{w}^{\mathopen{}\left(k\right)\mathclose{}}- \mathbf{y}\right\|\mathclose{}_2^2$ approaches zero. This implies that for any fixed $\sigma_{\Psi}>0$, we obtain:
\begin{multline}
\lim_{k\rightarrow\infty}\bar{\mathcal{L}}_{\sigma{\Psi}^{\mathopen{}\left(k\right)\mathclose{}},\mu^{\mathopen{}\left(k\right)\mathclose{}}}\mathopen{}\left(\mathbf{w}^*,\mathbf{z}^*,\boldsymbol{\Psi}^*\right)\mathclose{}=\\ \lim_{k\rightarrow\infty}\bar{\mathcal{L}}_{\sigma_{\Psi},\mu^{\mathopen{}\left(k\right)\mathclose{}}}\mathopen{}\left(\mathbf{w}^*,\mathbf{z}^*,\boldsymbol{\Psi}^*\right)\mathclose{}.
\end{multline}

Moreover, considering that $\lim_{\mu \rightarrow 0} h(\mathbf{z}^,\mu)=\rho_{\tau}(\mathbf{z}^)$, it follows:
\begin{multline}
    \lim_{k\rightarrow\infty}\bar{\mathcal{L}}_{\sigma_{\Psi},\mu^{\mathopen{}\left(k\right)\mathclose{}}}\mathopen{}\left(\mathbf{w}^*,\mathbf{z}^*,\boldsymbol{\Psi}^*\right)\mathclose{}=\\ \lim_{k\rightarrow\infty}{\mathcal{L}}_{\sigma_{\Psi}}\mathopen{}\left(\mathbf{w}^*,\mathbf{z}^*,\boldsymbol{\Psi}^*\right)\mathclose{}
\end{multline}

Finnaly, Theorem \ref{lemma13} assures that $\lim_{k\rightarrow\infty}\|\textbf{d}^{k_j}\|=0$. Hence, we can conclude that $\textbf{0}$ is in the sub-differential of $\mathcal{L}_{\sigma_{\Psi}}\mathopen{}\left(\mathbf{w}^*,\mathbf{z}^*,\boldsymbol{\Psi}^*\right)\mathclose{}$, indicating that $\mathopen{}\left(\mathbf{w}^*, \mathbf{z}^*, \boldsymbol{\Psi}^*\right)\mathclose{}$ is indeed a stationary point of ${\mathcal{L}}_{\sigma_{\Psi}}\mathopen{}\left(\mathbf{w}^*,\mathbf{z}^*,\boldsymbol{\Psi}^*\right)\mathclose{}$ and satisfies the KKT conditions.
This completes the proof.
\end{proof}
\begin{figure*}[ht]
     \centering
     \begin{subfigure}[b]{0.49\textwidth}
         \centering
        {\includegraphics[width=80mm, height=60mm]{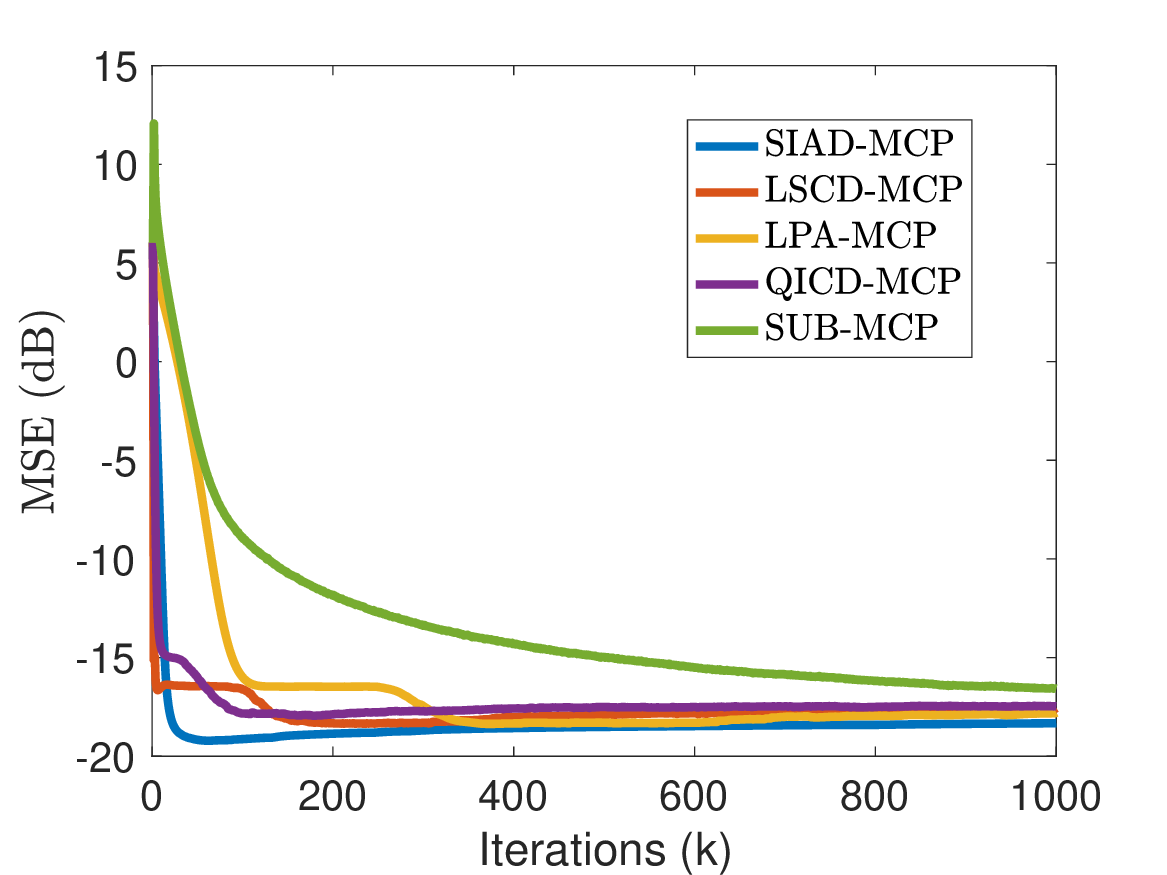}  \vspace{-1mm} }
      \caption{MCP ($\tau=0.55$)}   
     \end{subfigure}
     \hfill
     \begin{subfigure}[b]{0.49\textwidth}
         \centering  {\includegraphics[width=80mm, height=60mm]{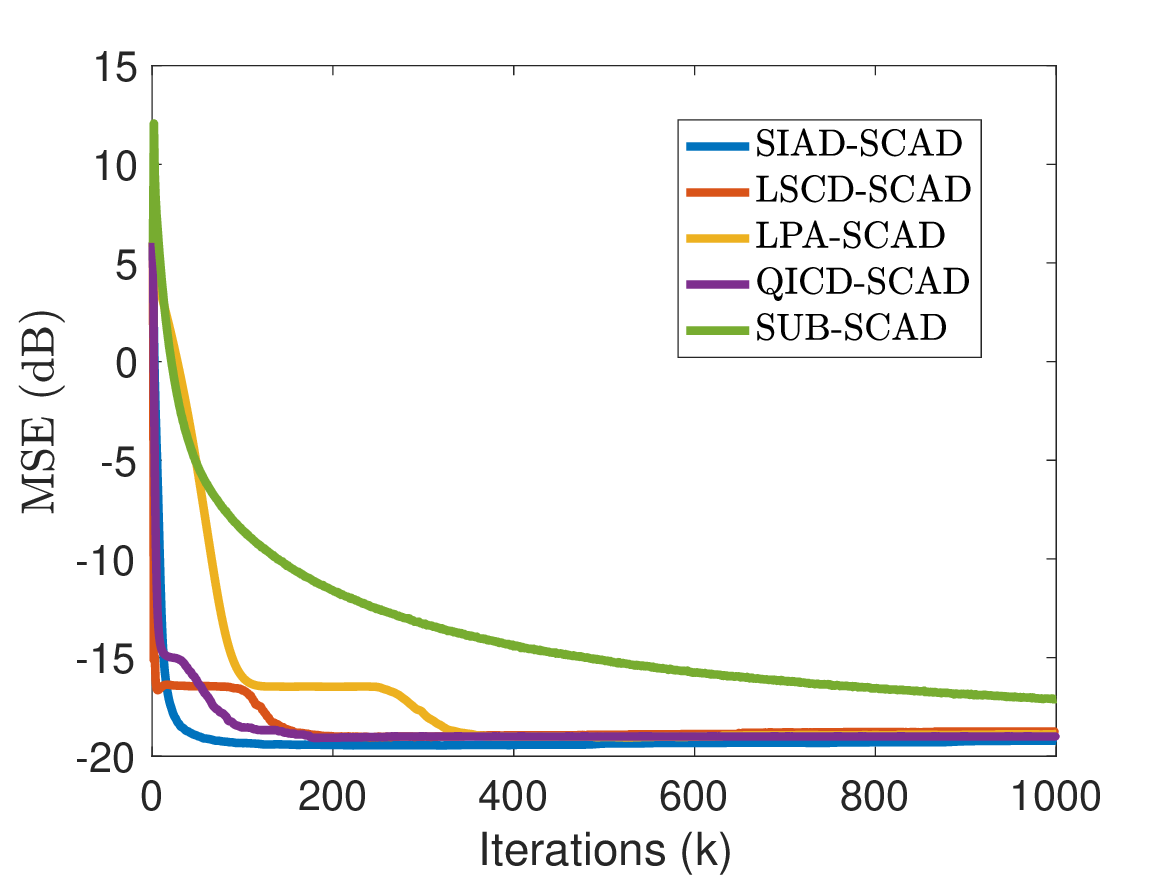}}
        \caption{SCAD ($\tau=0.55$)} 
     \end{subfigure}
    \centering
     \begin{subfigure}[b]{0.49\textwidth}
         \centering
        {\includegraphics[width=80mm, height=60mm]{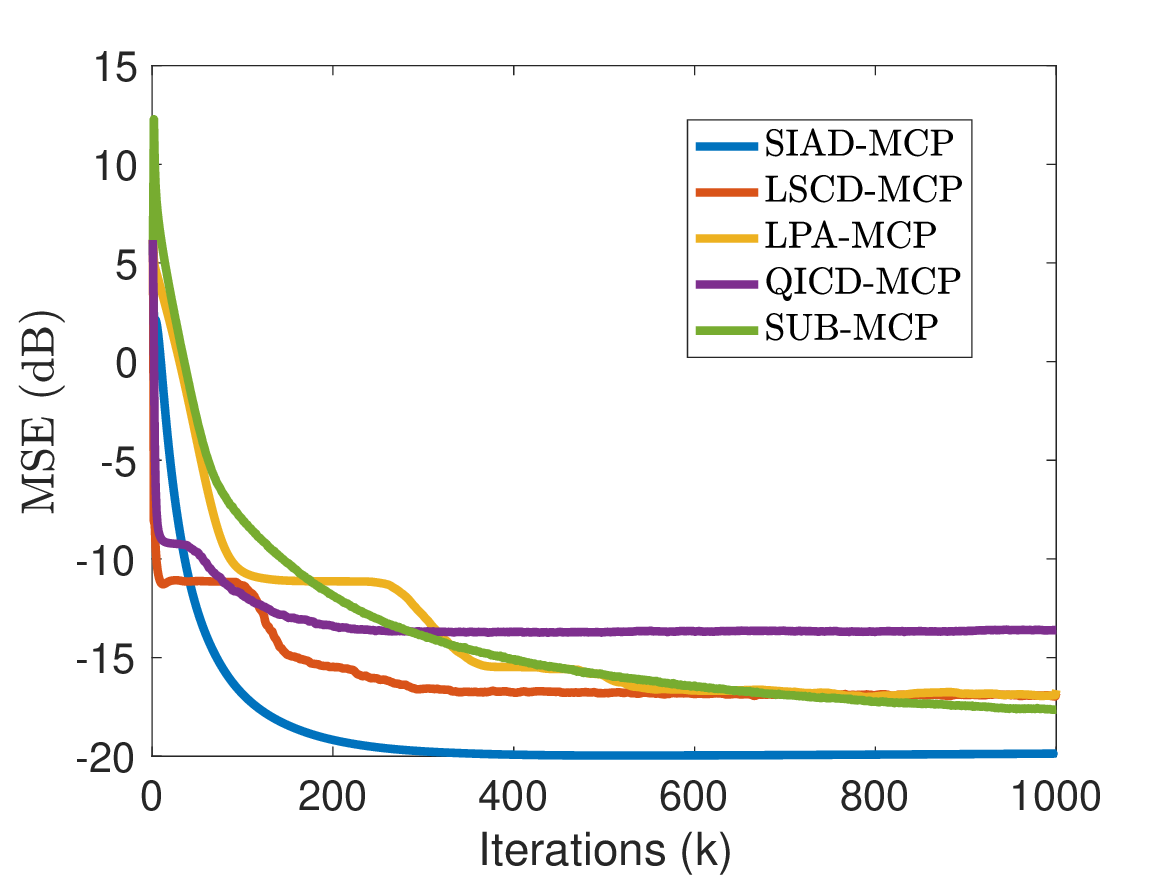}  \vspace{-1mm} }
      \caption{MCP ($\tau=0.7$)}   
     \end{subfigure}
     \hfill
     \begin{subfigure}[b]{0.49\textwidth}
         \centering  {\includegraphics[width=80mm, height=60mm]{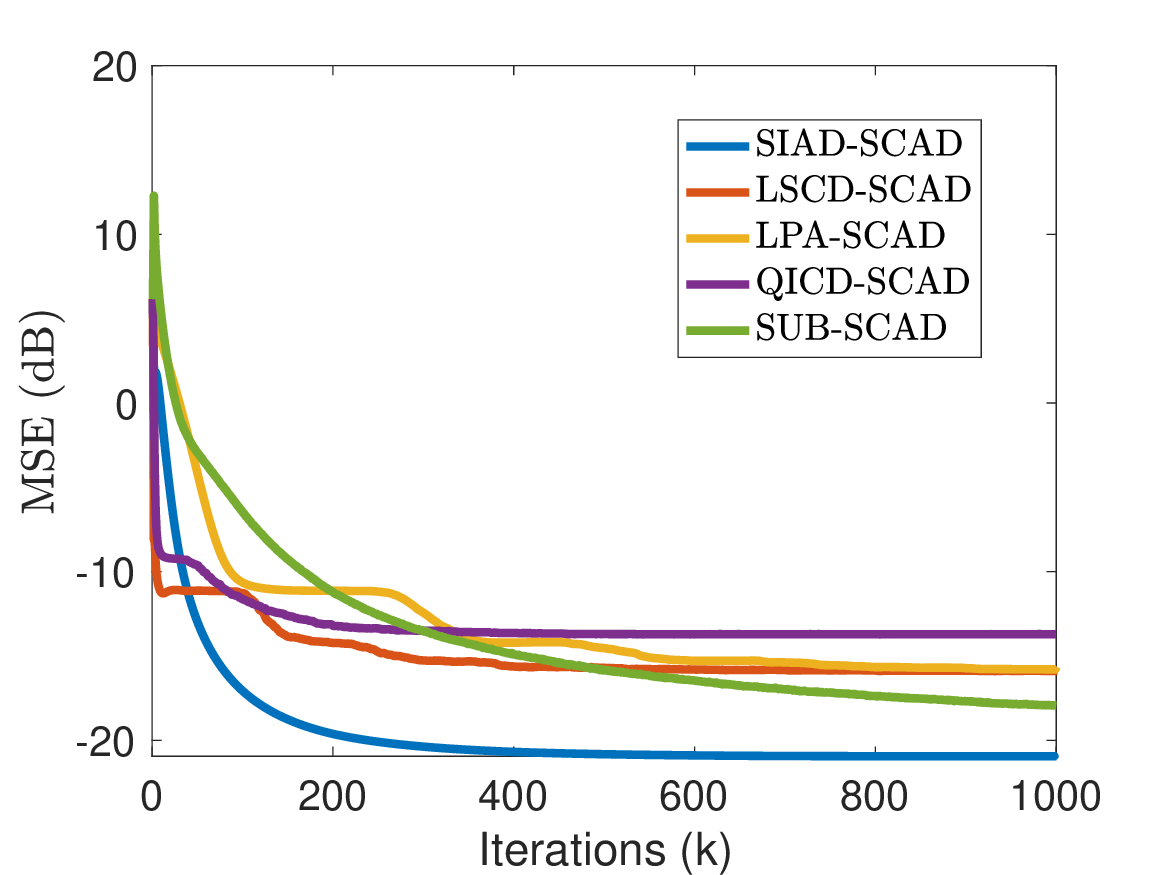}}
        \caption{SCAD ($\tau=0.7$)} 
     \end{subfigure}
     \centering
   \caption{MSE versus iterations}
   \label{fig1}
\end{figure*}

\begin{figure*}[ht]
     \centering
       \begin{subfigure}[b]{0.49\textwidth}
         \centering
        {\includegraphics[width=80mm, height=60mm]{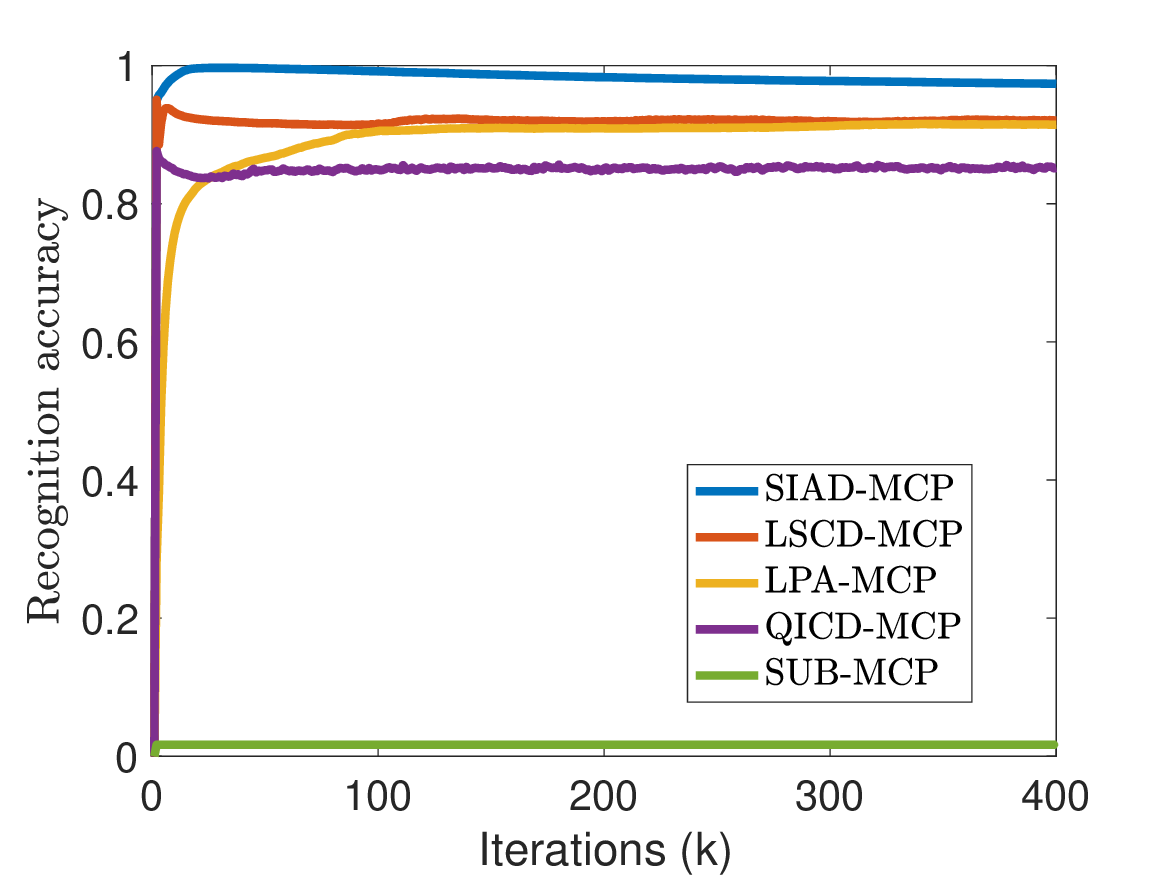}  \vspace{-1mm} }
      \caption{MCP ($\tau=0.55$)}   
     \end{subfigure}
     \hfill
     \begin{subfigure}[b]{0.49\textwidth}
         \centering  {\includegraphics[width=80mm, height=60mm]{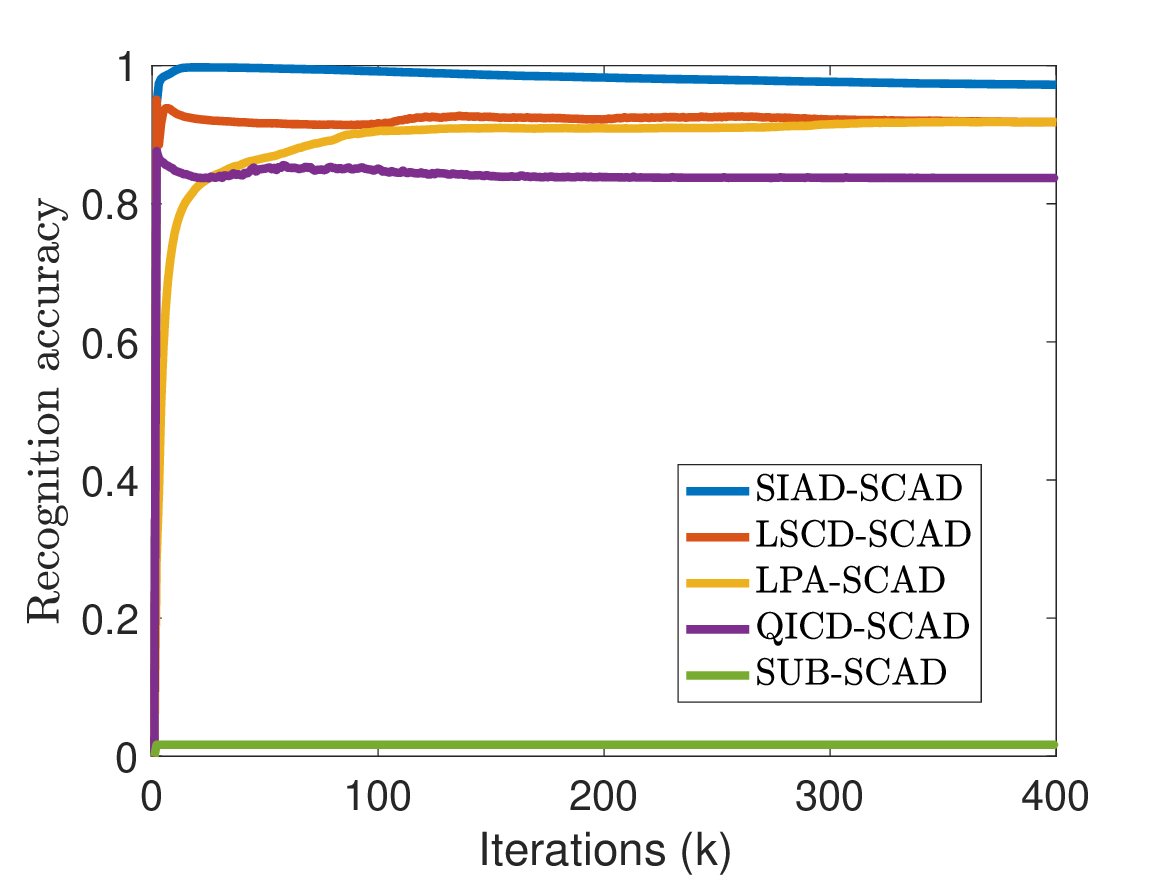}}
        \caption{SCAD ($\tau=0.55$)} 
     \end{subfigure}
     \centering
    \begin{subfigure}[b]{0.49\textwidth}
         \centering
        {\includegraphics[width=80mm, height=60mm]{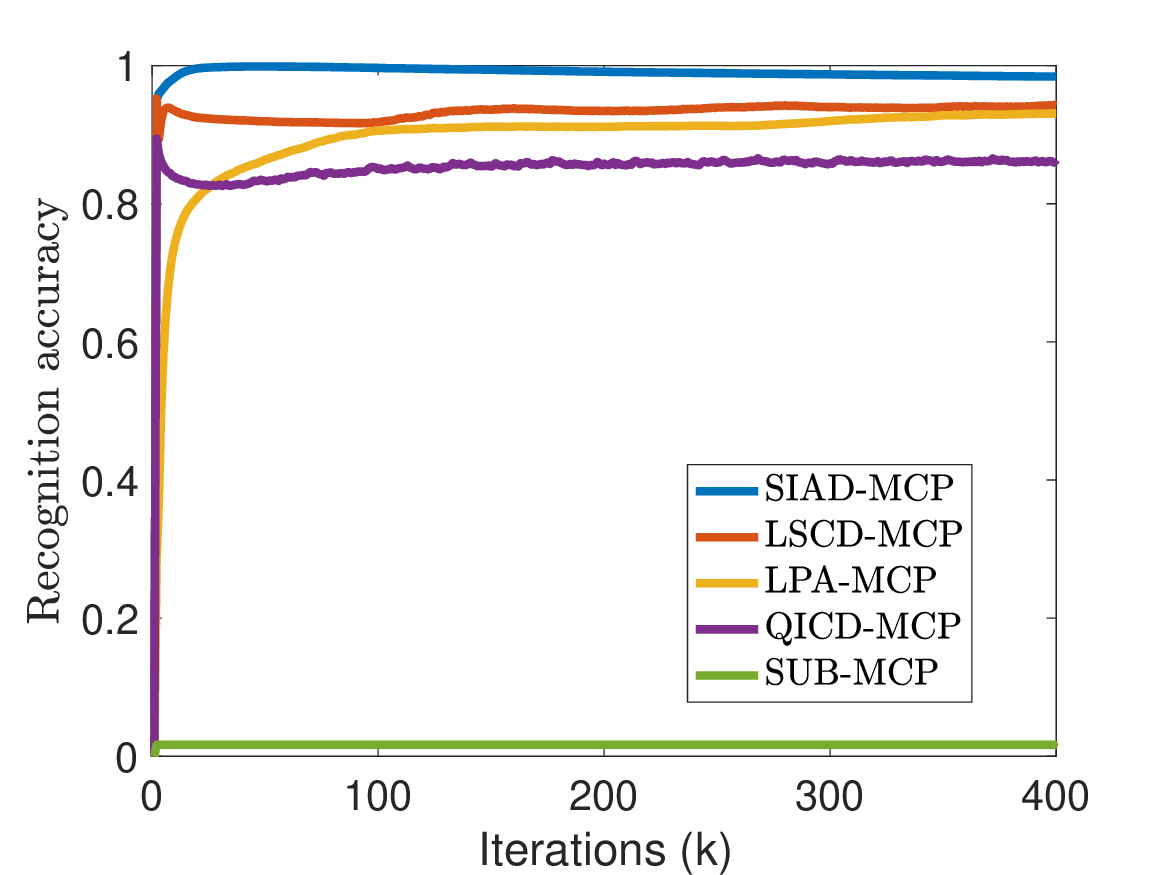}  \vspace{-1mm} }
      \caption{MCP ($\tau=0.7$)}   
     \end{subfigure}
     \hfill
     \begin{subfigure}[b]{0.49\textwidth}
         \centering  {\includegraphics[width=80mm, height=60mm]{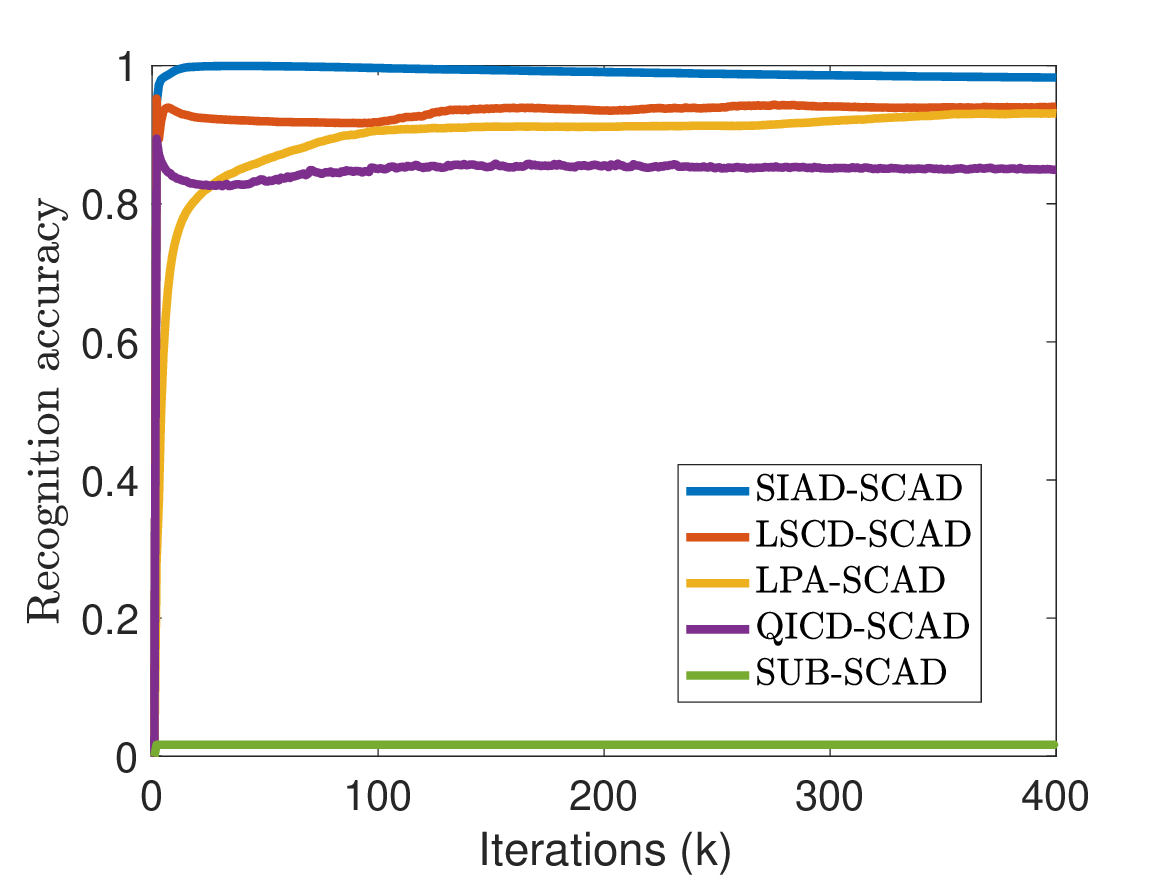}}
        \caption{SCAD ($\tau=0.7$)}
     \end{subfigure}
     \centering
   \caption{Accuracy of correctly recognizing active and non-active coefficients}
   \label{fig2}
\end{figure*}

\begin{figure*}[ht]
     \centering
     \begin{subfigure}[b]{0.49\textwidth}
         \centering
        {\includegraphics[width=80mm, height=60mm]{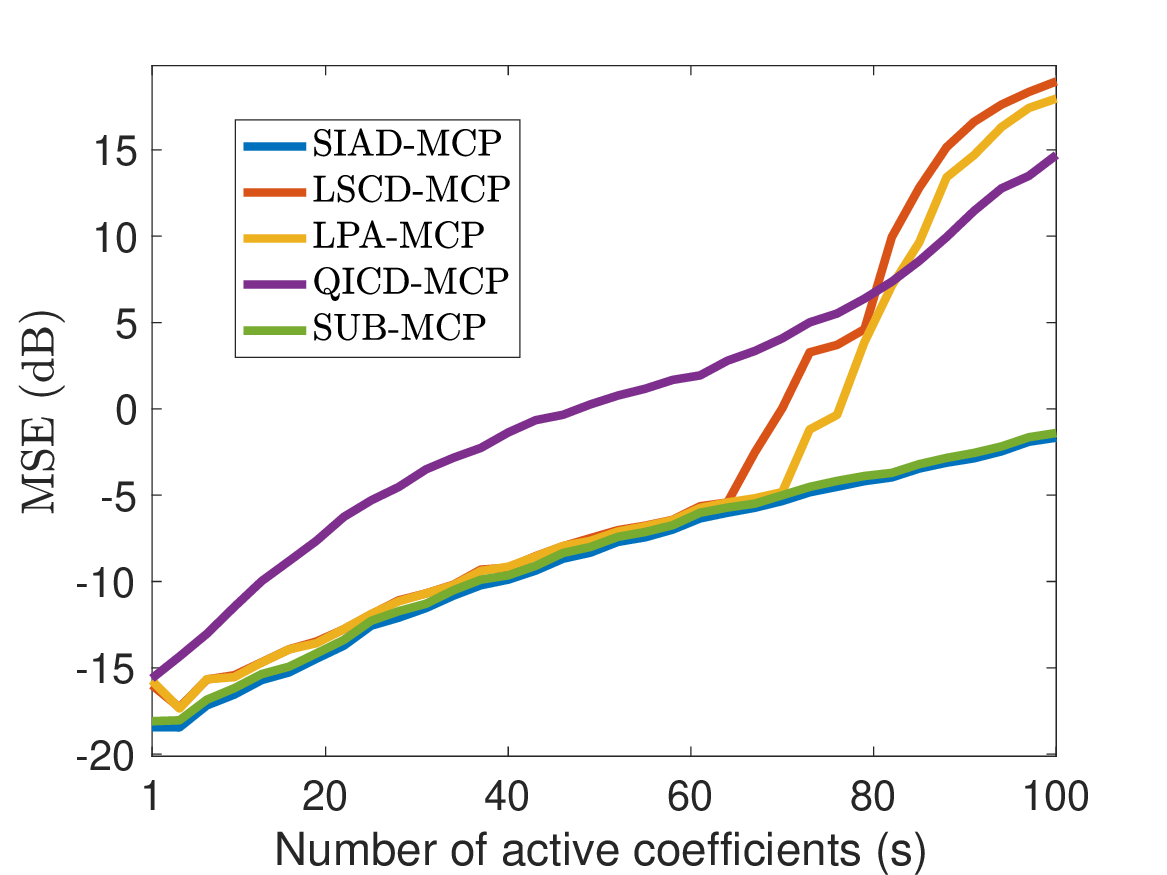}  \vspace{-1mm} }
      \caption{MCP}   
     \end{subfigure}
     \hfill
     \begin{subfigure}[b]{0.49\textwidth}
         \centering  {\includegraphics[width=80mm, height=60mm]{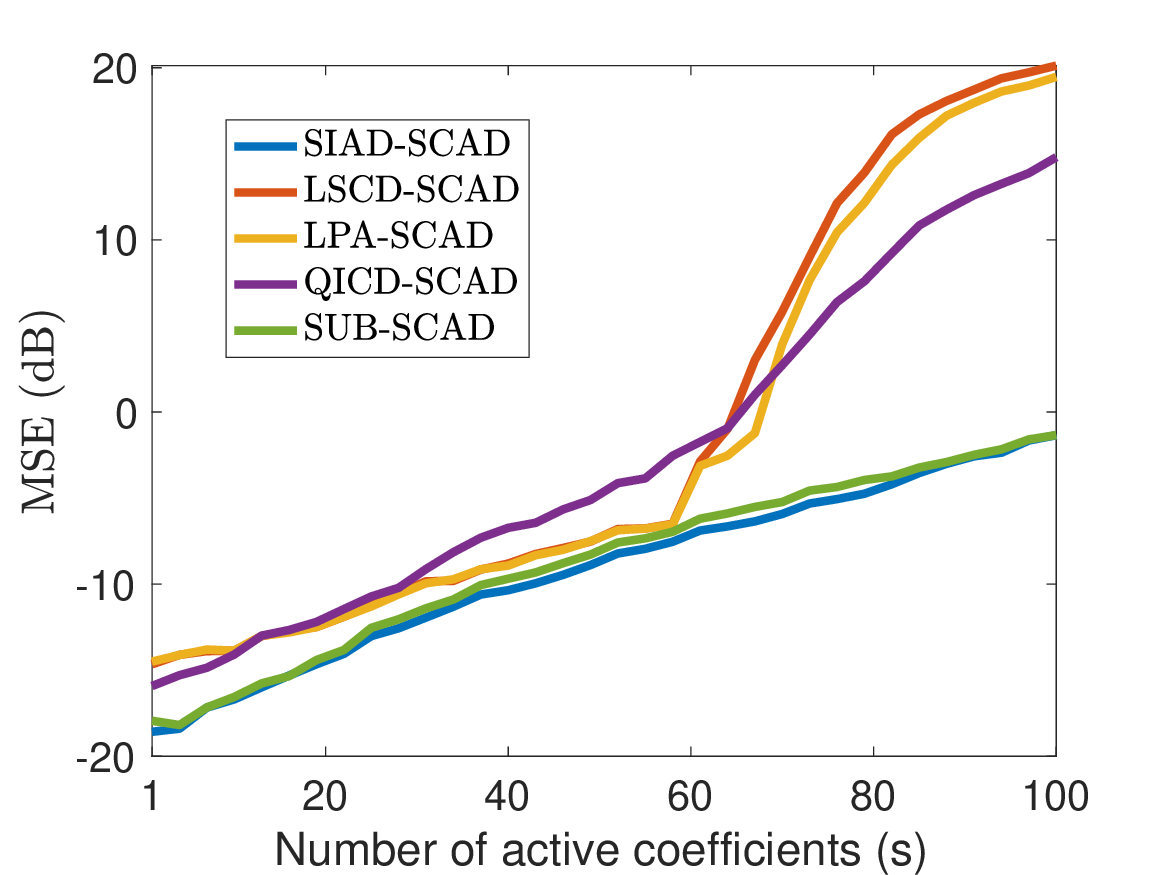}}
        \caption{SCAD} 
     \end{subfigure}
     \centering
   \caption{MSE versus the number of active coefficients $s$ in model parameter $\boldsymbol{\beta}_{\tau} \in \mathbb{R}^P$.}
   \label{fig3}
\end{figure*}
\begin{figure*}[ht]
     \centering
     \begin{subfigure}[b]{0.49\textwidth}
         \centering
        {\includegraphics[width=80mm, height=60mm]{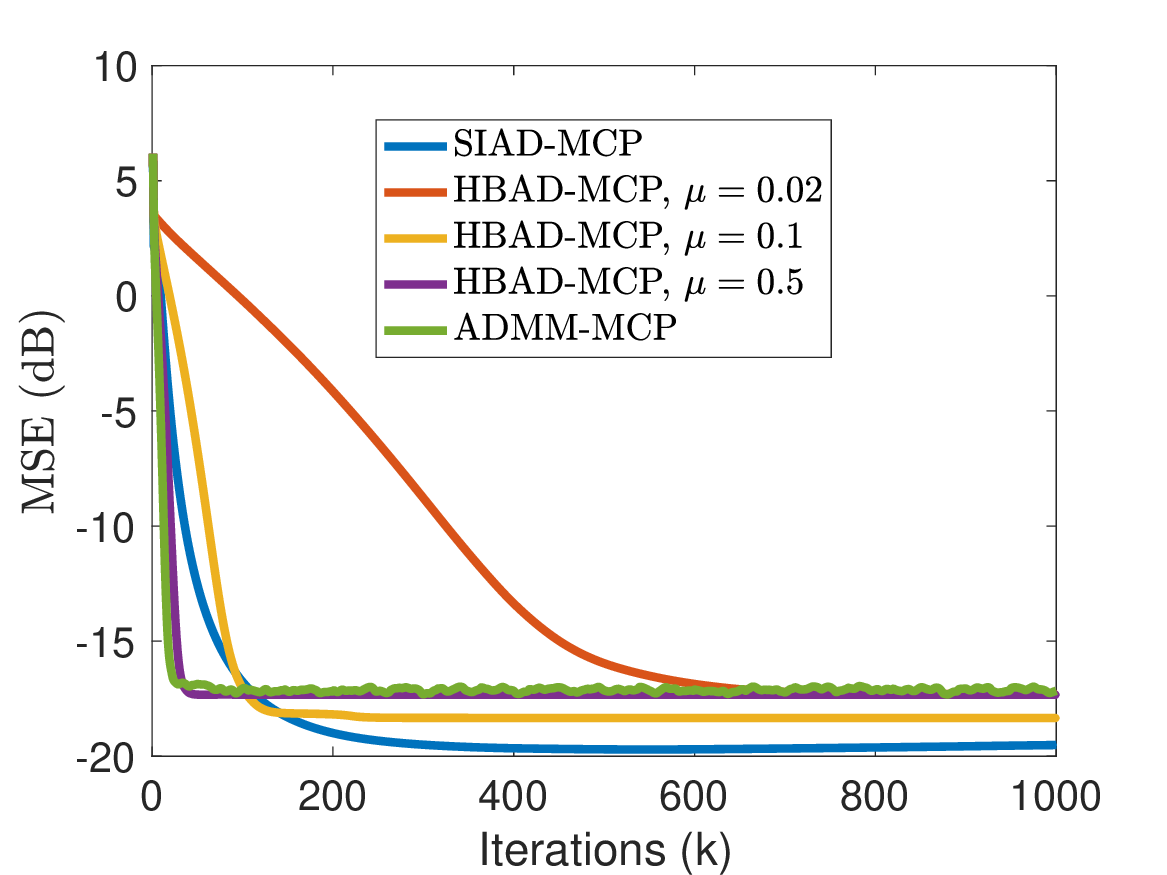}  \vspace{-1mm} }
      \caption{MCP}   
     \end{subfigure}
     \hfill
     \begin{subfigure}[b]{0.49\textwidth}
         \centering  {\includegraphics[width=80mm, height=60mm]{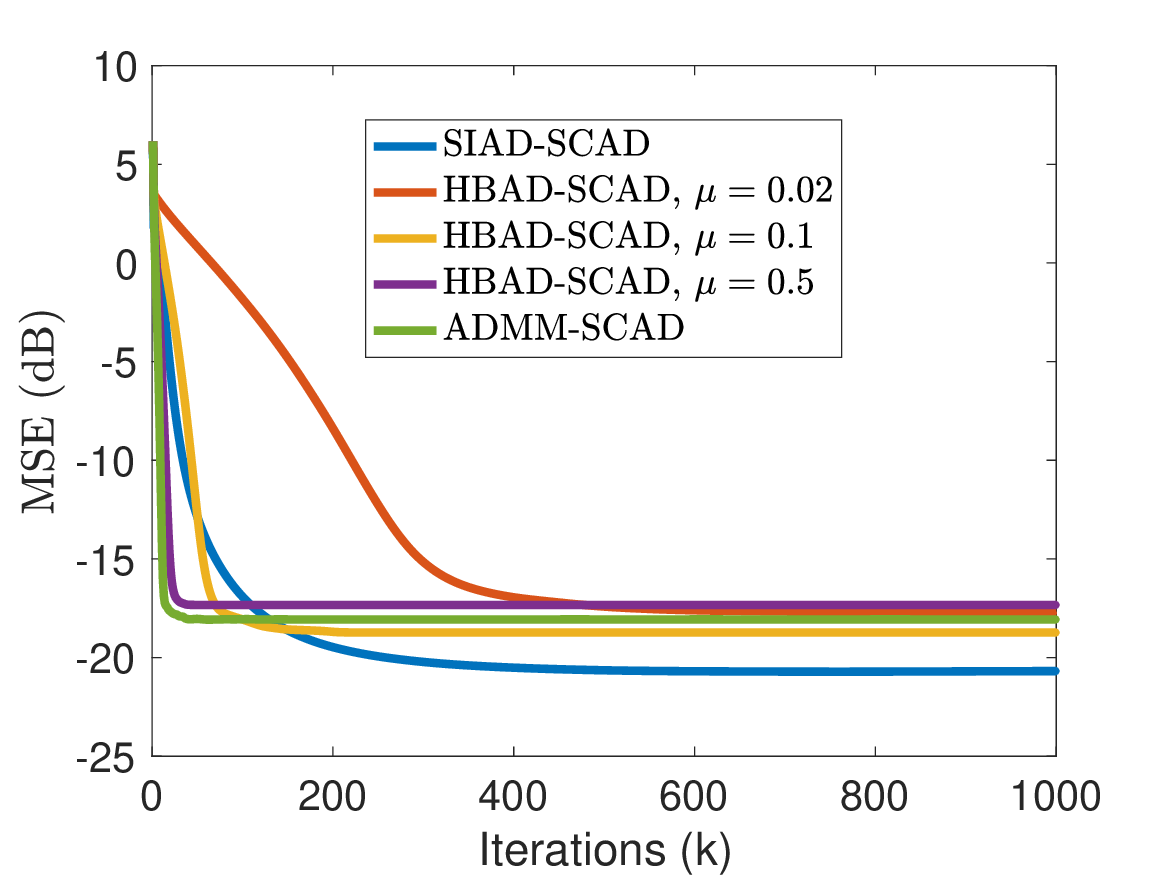}}
        \caption{SCAD} 
     \end{subfigure}
     \centering
   \caption{SIAD vs ADMM and HBAD}
   \label{fig4}
\end{figure*}
\section{Simulation results}\label{sec6}
In this section, we present a comprehensive simulation study to evaluate the performance of the proposed smoothing time-increasing penalty ADMM (SIAD) algorithm in the context of sparse quantile regression. We compare the SIAD algorithm with existing state-of-the-art approaches, including QICD \cite{peng2015iterative}, LPA \cite{gu2018admm},  LSCD \cite{gu2018admm}, also sub-gradient method \cite{mirzaeifard2023distributed}. The performance of these algorithms is assessed in terms of convergence rate, efficiency in terms of mean square error (MSE), and accuracy in recognizing active and non-active coefficients.  Furthermore, we investigate the efficiency and convergence property of the SIAD algorithm compared to conventional ADMM, ADMM without smoothing technique, and Huber loss for quantile regression with ADMM (HBAD), which is the SIAD algorithm with fixed $\mu$ and $\sigma_{\lambda}$, in terms of MSE.
\subsection{Simulation Setup}
For all four scenarios, we fix the penalty parameters $\gamma_{\text{SCAD}}=3.1$, $\gamma_{\text{MCP}}=2.1$, $\sigma_{\Psi}^{\mathopen{}\left(0\right)\mathclose{}}=1$, $\beta=\sqrt{3}$, $c=0.5$ and $\lambda=0.055$. The simulation results are obtained by averaging over $100$ independent trials.  The performance measure considered in this study is the mean square error (MSE), defined as $\text{E}\mathopen{}\left\|\mathbf{\hat{w}}-\mathbf{w}\right\|\mathclose{}_2^2$. We assess the algorithms in terms of convergence speed, efficiency in terms of MSE, and accuracy in correctly recognizing active and non-active coefficients.  First, we generate an observation model as follows. Let $\mathopen{}\left(\tilde{\mathbf{x}}_1,\tilde{\mathbf{x}}_2, \ldots, \tilde{\mathbf{x}}_P \right)\mathclose{}^{\text{T}} \sim \mathcal{N}\mathopen{}\left(0,\boldsymbol{\Sigma}\right)\mathclose{}$, where $\boldsymbol{\Sigma}_{pq}=0.5^{|p-q|}$. We set  $\mathbf{x}_1=\Phi\mathopen{}\left(\tilde{\mathbf{x}}_1\right)\mathclose{}$ and $\mathbf{x}_p =\tilde{\mathbf{x}}_p$ for $p= 2, 3, \ldots , P$, where $\Phi\mathopen{}\left(\cdot\right)\mathclose{}$ is the cumulative distribution function of $\mathcal{N}\mathopen{}\left(0,1\right)\mathclose{}$. 

For scenarios one, two and four, the observation model for generating data is given by:
\begin{align}
    \mathbf{y}=\sum_{p=1}^{P}  \xi_p \mathbf{x}_p+ \mathbf{x}_6+\mathbf{x}_{12}+\mathbf{x}_{15}+\mathbf{x}_{20}+0.7 \boldsymbol{\epsilon} \mathbf{x}_1,
\end{align}
where $\boldsymbol{\epsilon}\distas{\text{i.i.d}} \mathcal{N}\mathopen{}\left(0,1\right)\mathclose{}$, and $ \xi_p \distas{\text{i.i.d}} \mathcal{N}\mathopen{}\left(0,10^{-6}\right)\mathclose{}$. Under these settings, the model to be estimated is a compressible system \cite{lima2014sparsity}. The $\tau$-th conditional quantile linear function can be achieved by  $\sum_{p=1}^{P}  \xi_p \mathbf{x}_p+\mathbf{x}_6+\mathbf{x}_{12}+\mathbf{x}_{15}+\mathbf{x}_{20}+0.7 \Phi\mathopen{}\left(\tau\right)\mathclose{}^{-1} \mathbf{x}_1$. We assume  $\mathopen{}\left(n,P\right)\mathclose{}=\mathopen{}\left(100,300\right)\mathclose{}$  and $\tau=0.7$. For the first and second  scenarios, we also examine $\tau=0.55$.

For the third scenario, we modify the observation model for generating data as follows:
\begin{align}
    \mathbf{y}=\sum_{p=1}^{P}  \xi_p \mathbf{x}_p+ \sum_{i \in \mathcal{M}}\mathbf{x}_i +0.7 \boldsymbol{\epsilon} \mathbf{x}_1,
\end{align}
with $\mathcal{M}\in \{2,\ldots,P\}$, $\boldsymbol{\epsilon}\distas{\text{i.i.d}} \mathcal{N}\mathopen{}\left(0,1\right)\mathclose{}$, and $ \xi_p \distas{\text{i.i.d}} \mathcal{N}\mathopen{}\left(0,10^{-6}\right)\mathclose{}$. The model to be estimated in the third scenario is also a compressible system \cite{lima2014sparsity}, but with a different structure compared to the first two scenarios. The $\tau$-th conditional quantile linear function can be achieved by $\sum_{p=1}^{P}  \xi_p \mathbf{x}_p+ \sum_{i \in \mathcal{M}}\mathbf{x}_i +0.7 \Phi\mathopen{}\left(\tau\right)\mathclose{}^{-1} \mathbf{x}_1$. We assume  $\mathopen{}\left(n,P\right)\mathclose{}=\mathopen{}\left(200,100\right)\mathclose{}$  and $\tau=0.7$.
\subsection{Results}
In the first scenario, the algorithms were compared in terms of their convergence speed and efficiency using MSE as the performance measure. The learning curves (MSE vs iterations) for the algorithms are shown in Fig. \ref{fig1} for $\tau={0.55,0.7}$. Fig. \ref{fig1} demonstrates that the proposed SIAD algorithm achieves a lower MSE than the other existing approaches regardless of the penalty function and the value of $\tau$. Additionally, the SIAD algorithm exhibits a faster convergence rate than other algorithms.

In the second scenario, the algorithms were compared based on the accuracy of recognizing active and non-active coefficients correctly. The accuracy measure is defined as the ratio of the number of active and non-active coefficients correctly identified to the total number of coefficients. Fig. \ref{fig2} shows the accuracy vs. iterations for the algorithms. Fig. \ref{fig2} illustrates that the SIAD algorithm can distinguish active and non-active coefficients more accurately compared to other methods, and as a result, SIAD achieves a better result for parameter selection.

In the third scenario, the robustness of the algorithms under different levels of sparsity was compared. As the number of active coefficients increased from $1$ to $P$, the algorithms were evaluated based on their MSE performance after $30000$ iterations. Fig. \ref{fig3} illustrates the MSE vs the number of active coefficients for all algorithms. From Fig. \ref{fig3}, it can be observed that the proposed SIAD algorithm performs consistently against all sparsity levels, ranging from highly sparse to non-sparse. In contrast, only the SUB algorithm, which is also a single-loop algorithm, performed similarly, and other state-of-the-art approaches exhibit poor performance when the sparsity level varies from moderately sparse to non-sparse.

In the fourth scenario, the convergence property of the SIAD algorithm was assessed in comparison to HBAD and conventional ADMM. We considered HBAD with three different smoothing parameters $\mu={0.02,0.1,0.5}$ and appropriate penalty parameters to ensure convergence. Fig. \ref{fig4} exhibits the learning curves for the algorithms. As one can specifically observe for the MCP penalty function, ADMM is fluctuating and does not converge, resulting in worse performance in terms of MSE. Although HBAD converges to a stationary point, the result of SIAD is slightly better than HBAD with the pre-examined smoothing parameter.

\section{conclusions}\label{sec7}
In this paper, a novel smoothing-based ADMM algorithm with time-increasing penalty parameters  has been proposed for the quantile regression penalized with non-convex and non-smooth sparse penalties. With our novel analysis, the convergence proof for the proposed algorithm has been conducted. The simulation results demonstrated that this single-loop ADMM algorithm could achieve better MSE than the QICD method and the LLA framework. Also, this algorithm performs consistently against all sparsity levels, especially in moderately sparse or non-sparse, where other algorithms had shown worse results.

\bibliographystyle{IEEEtran}
\bibliography{strings}

\begin{thebibliography}{10}
\providecommand{\url}[1]{#1}
\csname url@samestyle\endcsname
\providecommand{\newblock}{\relax}
\providecommand{\bibinfo}[2]{#2}
\providecommand{\BIBentrySTDinterwordspacing}{\spaceskip=0pt\relax}
\providecommand{\BIBentryALTinterwordstretchfactor}{4}
\providecommand{\BIBentryALTinterwordspacing}{\spaceskip=\fontdimen2\font plus
\BIBentryALTinterwordstretchfactor\fontdimen3\font minus
  \fontdimen4\font\relax}
\providecommand{\BIBforeignlanguage}[2]{{%
\expandafter\ifx\csname l@#1\endcsname\relax
\typeout{** WARNING: IEEEtran.bst: No hyphenation pattern has been}%
\typeout{** loaded for the language `#1'. Using the pattern for}%
\typeout{** the default language instead.}%
\else
\language=\csname l@#1\endcsname
\fi
#2}}
\providecommand{\BIBdecl}{\relax}
\BIBdecl

\bibitem{seber2012linear}
G.~A. Seber and A.~J. Lee, \emph{Linear Regression Analysis}.\hskip 1em plus
  0.5em minus 0.4em\relax John Wiley \& Sons, Feb. 2012, vol. 329.

\bibitem{yu2020probabilistic}
Y.~Yu, X.~Han, M.~Yang, and J.~Yang, ``Probabilistic prediction of regional
  wind power based on spatiotemporal quantile regression,'' \emph{IEEE
  Transactions on Industry Applications}, vol.~56, no.~6, pp. 6117--6127, Dec.
  2020.

\bibitem{taieb2016fore}
S.~Ben~Taieb, R.~Huser, R.~J. Hyndman, and M.~G. Genton, ``Forecasting
  uncertainty in electricity smart meter data by boosting additive quantile
  regression,'' \emph{IEEE Transactions on Smart Grid}, vol.~7, no.~5, pp.
  2448--2455, Mar. 2016.

\bibitem{happy2021stat}
H.~Aprillia, H.-T. Yang, and C.-M. Huang, ``Statistical load forecasting using
  optimal quantile regression random forest and risk assessment index,''
  \emph{IEEE Transactions on Smart Grid}, vol.~12, no.~2, pp. 1467--1480, Oct.
  2021.

\bibitem{he2016regularized}
Q.~He, L.~Kong, Y.~Wang, S.~Wang, T.~A. Chan, and E.~Holland, ``Regularized
  quantile regression under heterogeneous sparsity with application to
  quantitative genetic traits,'' \emph{Computational Statistics \& Data
  Analysis}, vol.~95, pp. 222--239, Mar. 2016.

\bibitem{algamal2018gene}
Z.~Y. Algamal, R.~Alhamzawi, and H.~T.~M. Ali, ``Gene selection for microarray
  gene expression classification using bayesian lasso quantile regression,''
  \emph{Computers in biology and medicine}, vol.~97, pp. 145--152, June 2018.

\bibitem{tibshirani2014adaptive}
R.~J. TIBSHIRANI, ``Adaptive piecewise polynomial estimation via trend
  filtering,'' \emph{The Annals of Statistics}, vol.~42, no.~1, pp. 285--323,
  Feb. 2014.

\bibitem{chen2021quantile}
I.-C. Chen, S.~J. Bertke, and B.~D. Curwin, ``Quantile regression for exposure
  data with repeated measures in the presence of non-detects,'' \emph{Journal
  of exposure science \& environmental epidemiology}, vol.~31, no.~6, pp.
  1057--1066, Nov. 2021.

\bibitem{wu2009variable}
Y.~Wu and Y.~Liu, ``Variable selection in quantile regression,''
  \emph{Statistica Sinica}, vol.~19, no.~2, pp. 801--817, Apr. 2009.

\bibitem{xue2012positive}
L.~Xue, S.~Ma, and H.~Zou, ``Positive-definite l1-penalized estimation of large
  covariance matrices,'' \emph{Journal of the American Statistical
  Association}, vol. 107, no. 500, pp. 1480--1491, Dec. 2012.

\bibitem{fan2001variable}
J.~Fan and R.~Li, ``Variable selection via nonconcave penalized likelihood and
  its oracle properties,'' \emph{Journal of the American statistical
  Association}, vol.~96, no. 456, pp. 1348--1360, Dec. 2001.

\bibitem{zhang2010nearly}
C.-H. Zhang, ``Nearly unbiased variable selection under minimax concave
  penalty,'' \emph{The Annals of statistics}, vol.~38, no.~2, pp. 894--942,
  Apr. 2010.

\bibitem{belloni2011l1}
A.~Belloni and V.~Chernozhukov, ``$l_1$-penalized quantile regression in
  high-dimensional sparse models,'' \emph{The Annals of Statistics}, vol.~39,
  no.~1, pp. 82--130, Feb. 2011.

\bibitem{koenker2005frisch}
R.~Koenker and P.~Ng, ``A frisch-newton algorithm for sparse quantile
  regression,'' \emph{Acta Mathematicae Applicatae Sinica}, vol.~21, no.~2, pp.
  225--236, May 2005.

\bibitem{wang2017distributed}
H.~Wang and C.~Li, ``Distributed quantile regression over sensor networks,''
  \emph{IEEE Transactions on Signal and Information Processing over Networks},
  vol.~4, no.~2, pp. 338--348, Apr. 2017.

\bibitem{ouyang2021lower}
Y.~Ouyang and Y.~Xu, ``Lower complexity bounds of first-order methods for
  convex-concave bilinear saddle-point problems,'' \emph{Mathematical
  Programming}, vol. 185, no. 1-2, pp. 1--35, Jan. 2021.

\bibitem{boyd2011distributed}
S.~Boyd, N.~Parikh, and E.~Chu, \emph{Distributed Optimization and Statistical
  Learning via the Alternating Direction Method of Multipliers}.\hskip 1em plus
  0.5em minus 0.4em\relax Now Publishers Inc, July 2011.

\bibitem{gu2018admm}
Y.~Gu, J.~Fan, L.~Kong, S.~Ma, and H.~Zou, ``{ADMM} for high-dimensional sparse
  penalized quantile regression,'' \emph{Technometrics}, vol.~60, no.~3, pp.
  319--331, July 2018.

\bibitem{wang2023smoothing}
L.~Wang and X.~Liu, ``Smoothing gradient tracking for decentralized
  optimization over the stiefel manifold with non-smooth regularizers,''
  \emph{arXiv preprint arXiv:2303.15882}, 2023.

\bibitem{zhao2023accelerated}
Y.~Zhao, X.~Liao, and X.~He, ``Accelerated projection algorithm based on
  smoothing approximation for distributed non-smooth optimization,'' \emph{IEEE
  Transactions on Control of Network Systems}, Jan. 2023.

\bibitem{chen2021distributed}
S.~Chen, A.~Garcia, and S.~Shahrampour, ``On distributed nonconvex
  optimization: Projected subgradient method for weakly convex problems in
  networks,'' \emph{IEEE Transactions on Automatic Control}, vol.~67, no.~2,
  pp. 662--675, Feb. 2021.

\bibitem{zeng2022moreau}
J.~Zeng, W.~Yin, and D.-X. Zhou, ``Moreau envelope augmented lagrangian method
  for nonconvex optimization with linear constraints,'' \emph{Journal of
  Scientific Computing}, vol.~91, no.~2, p.~61, may 2022.

\bibitem{wang2019global}
Y.~Wang, W.~Yin, and J.~Zeng, ``Global convergence of {ADMM} in nonconvex
  nonsmooth optimization,'' \emph{Journal of Scientific Computing}, vol.~78,
  no.~1, pp. 29--63, Jan. 2019.

\bibitem{yashtini2020convergence}
M.~Yashtini, ``Convergence analysis of a variable metric proximal linearized
  {ADMM} with over-relaxation parameter in nonconvex nonsmooth optimization,''
  \emph{arXiv preprint arXiv:2009.05361}, Sep. 2020.

\bibitem{davis2018subgradient}
D.~Davis, D.~Drusvyatskiy, K.~J. MacPhee, and C.~Paquette, ``Subgradient
  methods for sharp weakly convex functions,'' \emph{Journal of Optimization
  Theory and Applications}, vol. 179, pp. 962--982, Dec. 2018.

\bibitem{swenson2022distributed}
B.~Swenson, R.~Murray, H.~V. Poor, and S.~Kar, ``Distributed stochastic
  gradient descent: Nonconvexity, nonsmoothness, and convergence to local
  minima,'' \emph{The Journal of Machine Learning Research}, vol.~23, no.~1,
  pp. 14\,751--14\,812, Jan. 2022.

\bibitem{peng2015iterative}
B.~Peng and L.~Wang, ``An iterative coordinate descent algorithm for
  high-dimensional nonconvex penalized quantile regression,'' \emph{Journal of
  Computational and Graphical Statistics}, vol.~24, no.~3, pp. 676--694, July
  2015.

\bibitem{sun2016majorization}
Y.~Sun, P.~Babu, and D.~P. Palomar, ``Majorization-minimization algorithms in
  signal processing, communications, and machine learning,'' \emph{IEEE
  Transactions on Signal Processing}, vol.~65, no.~3, pp. 794--816, Aug. 2016.

\bibitem{davis2019stochastic}
D.~Davis and D.~Drusvyatskiy, ``Stochastic model-based minimization of weakly
  convex functions,'' \emph{SIAM Journal on Optimization}, vol.~29, no.~1, pp.
  207--239, Jan. 2019.

\bibitem{mirzaeifard2023distributed}
R.~Mirzaeifard, V.~C. Gogineni, N.~K.~D. Venkategowda, and S.~Werner,
  ``Distributed quantile regression with non-convex sparse penalties,'' in
  \emph{2023 IEEE Statistical Signal Processing Workshop (SSP)}, July 2023, pp.
  250--254.

\bibitem{hong2015convergence}
M.~Hong, Z.-Q. Luo, and M.~Razaviyayn, ``Convergence analysis of alternating
  direction method of multipliers for a family of nonconvex problems,'' in
  \emph{IEEE International Conference on Acoustics, Speech and Signal
  Processing}, Apr. 2015, pp. 3836--3840.

\bibitem{themelis2020douglas}
A.~Themelis and P.~Patrinos, ``Douglas--{Rachford} splitting and {ADMM} for
  nonconvex optimization: Tight convergence results,'' \emph{SIAM Journal on
  Optimization}, vol.~30, no.~1, pp. 149--181, Jan. 2020.

\bibitem{mirzaeifard2022robust}
R.~Mirzaeifard, N.~K. Venkategowda, and S.~Werner, ``Robust phase retrieval
  with non-convex penalties,'' in \emph{2022 56th Asilomar Conference on
  Signals, Systems, and Computers}.\hskip 1em plus 0.5em minus 0.4em\relax
  IEEE, Oct. 2022, pp. 1291--1295.

\bibitem{mirzaeifard2022admm}
R.~Mirzaeifard, N.~K. Venkategowda, V.~C. Gogineni, and S.~Werner, ``{ADMM} for
  sparse-penalized quantile regression with non-convex penalties,'' in
  \emph{2022 30th European Signal Processing Conference (EUSIPCO)}.\hskip 1em
  plus 0.5em minus 0.4em\relax IEEE, Aug. 2022, pp. 2046--2050.

\bibitem{koenker1982robust}
R.~Koenker and G.~Bassett~Jr, ``Robust tests for heteroscedasticity based on
  regression quantiles,'' \emph{Econometrica: Journal of the Econometric
  Society}, pp. 43--61, Jan. 1982.

\bibitem{varma2019vector}
R.~Varma, H.~Lee, J.~Kova{\v{c}}evi{\'c}, and Y.~Chi, ``Vector-valued graph
  trend filtering with non-convex penalties,'' \emph{IEEE Transactions on
  Signal and Information Processing over Networks}, vol.~6, pp. 48--62, Dec.
  2019.

\bibitem{poliquin1996prox}
R.~Poliquin and R.~Rockafellar, ``Prox-regular functions in variational
  analysis,'' \emph{Transactions of the American Mathematical Society}, vol.
  348, no.~5, pp. 1805--1838, May 1996.

\bibitem{chen2012smoothing}
X.~Chen, ``Smoothing methods for nonsmooth, nonconvex minimization,''
  \emph{Mathematical programming}, vol. 134, pp. 71--99, Aug. 2012.

\bibitem{bian2013neural}
W.~Bian and X.~Chen, ``Neural network for nonsmooth, nonconvex constrained
  minimization via smooth approximation,'' \emph{IEEE transactions on neural
  networks and learning systems}, vol.~25, no.~3, pp. 545--556, Oct. 2013.

\bibitem{huang2012selective}
J.~Huang, P.~Breheny, and S.~Ma, ``A selective review of group selection in
  high-dimensional models,'' \emph{Statistical Science}, vol.~27, no.~4, Oct.
  2012.

\bibitem{attouch2013convergence}
H.~Attouch, J.~Bolte, and B.~F. Svaiter, ``Convergence of descent methods for
  semi-algebraic and tame problems: proximal algorithms, forward--backward
  splitting, and regularized gauss--seidel methods,'' \emph{Mathematical
  Programming}, vol. 137, no. 1-2, pp. 91--129, Feb. 2013.

\bibitem{lima2014sparsity}
M.~V. Lima, T.~N. Ferreira, W.~A. Martins, and P.~S. Diniz, ``Sparsity-aware
  data-selective adaptive filters,'' \emph{IEEE Transactions on Signal
  Processing}, vol.~62, no.~17, pp. 4557--4572, Sep. 2014.

\end{thebibliography}
\end{document}